\RequirePackage{fix-cm}
\documentclass[smallextended]{svjour3}       
\smartqed  
\usepackage{amsmath}
\usepackage{graphicx,}
\usepackage{amsfonts,amssymb}
\usepackage{subfigure,booktabs,multirow}
\usepackage{arydshln}
\usepackage[figuresright]{rotating}
\usepackage{mathrsfs,enumerate,color,bm}
\usepackage{algorithm,algpseudocode}
\usepackage{booktabs}
\usepackage{color}
\usepackage{colortbl}
\usepackage{xcolor}
\usepackage{graphicx}
\usepackage[colorlinks,linkcolor=red,anchorcolor=gray,citecolor=green,urlcolor=black]{hyperref}
\usepackage{cite}

\newcommand{\R}{\mathbb{R}}

\newcommand{\rank}{\operatorname{rank}}

\graphicspath{{figs/}}

\begin{document}

\title{Nonnegative Low Rank
Tensor Approximation with Applications
to Multi-dimensional Images
}


\author{Tai-Xiang Jiang        \and
        Michael K. Ng\and
        Junjun Pan\and
        Guang-Jing Song 
}


\institute{T.-X. Jiang \at
              School of Economic Information Engineering, Southwestern University of Finance and Economics, Chengdu. \\
              \email{taixiangjiang@gmail.com}
           \and
           Michael K. Ng \at
              Department of Mathematics, The University of Hong Kong, Pokfulam, Hong Kong.\\
              \email{mng@maths.hku.hk}
              \and
           Junjun Pan \at
              Department of Mathematics, The University of Hong Kong, Pokfulam, Hong Kong.\\
              \email{junjpan@hku.hk}
              \and
           Guang-Jing Song \at
              School of Mathematics and Information Sciences, Weifang University, Weifang 261061, P.R. China.\\
              \email{sgjshu@163.com}
}

\date{Received: date / Accepted: date}

\maketitle

\begin{abstract}
The main aim of this paper is to develop a new algorithm for computing nonnegative low rank tensor approximation for nonnegative tensors that arise in many multi-dimensional imaging applications. Nonnegativity is one of the important property as
each pixel value refers to nonzero light intensity in image data acquisition.
Our approach is different from classical nonnegative tensor factorization (NTF) which requires each factorized matrix and/or tensor to be
nonnegative. In this paper, we determine a nonnegative low Tucker rank
tensor to approximate a given nonnegative tensor.
We propose an alternating projections algorithm
for computing such nonnegative low rank tensor approximation, which is referred to as NLRT.
The convergence of the proposed manifold projection method is established.
Experimental results for synthetic data and multi-dimensional images
are presented to demonstrate the performance of NLRT is better than
state-of-the-art NTF methods.

\end{abstract}

\section{Introduction}

Nonnegative data is very common in many data analysis applications. For instance, in image analysis, image pixel values are nonnegative and the associated images can be seen as nonnegative matrices for clustering and recognition tasks. When the data is already high dimensional by nature, for example, video data, hyperspectral data, fMRI data and so on,  it then seems more natural to represent the information in a high dimensional space, rather than flatten the data to a matrix. The data represented in high dimension is referred to as a tensor.

An $m$-dimensional tensor $\mathcal{A}$ is a multi-dimensional array,
$\mathcal{A}\in \mathbb{R}^{n_1\times \cdots\times n_m}$. To extract pertinent information from a given large tensor data, low rank tensor decompositions are usually considered. In recent decades, various of tensor decompositions have been developed according to different applications. The most famous
and widely used decompositions are Canonical Polyadic decomposition (CPD) and Tucker decomposition. For more details  of tensor applications and tensor decompositions, we refer to the review papers \cite{kolda2009tensor,sidiropoulos2017tensor}.
 In this paper, we only target on tensor in a Tucker form. Hence, in the following, we will briefly review Tucker decomposition.

Given a tensor $\mathcal{A}\in \mathbb{R}^{n_1\times n_2\times \cdots \times n_m}$, the Tucker decomposition \cite{de2000multilinear, tucker1966some,kolda2009tensor} is defined as follows:
\begin{equation}\label{tucker}
\mathcal{A}=\mathcal{G}\times_1 \mathbf{U}^{(1)}\times_2 \mathbf{U}^{(2)}\times_3\cdots\times_m \mathbf{U}^{(m)},
\end{equation}
i.e.,
\begin{equation}
{\cal A}_{i_1,\cdots,i_m}=\sum_{j_1,\cdots,j_m}
{\cal G}_{j_1,\cdots,j_m} {\bf U}^{(1)}_{i_1,j_1} \cdots {\bf U}^{(m)}_{i_m,j_m},
\end{equation}
where
$\mathcal{G}=( {\cal G}_{j_1,j_2,\cdots,j_m})\in \mathbb{R}^{J_1\times J_2\times\cdots\times J_m}$,
${\bf U}^{(k)}$ is a $n_k$-by-$J_k$ matrix (whose columns are usually mutually orthogonal),
$\times_k$ denotes the $k$-mode matrix product of a tensor defined by
$$
(\mathcal{G}\times_k \mathbf{U}^{(k)})_{j_1\cdots j_{k-1}i_{k}j_{k+1}\cdots j_m}=\sum^{J_k}_{j_k=1} {\cal G}_{j_1\cdots j_{k-1}j_kj_{k+1}\cdots j_m} {\bf U}^{(k)}_{i_k,j_k}.
$$
The minimal value of $(J_1,J_2,\cdots,J_m)$ is defined as  Tucker (or multilinear) rank of ${\cal A}$, denoted as $\rank_{T}({\cal A})=(J_1,J_2,\cdots,J_m)$.

Since high-dimensional nonnegative data are everywhere in real world, and the nonnegativity of factor matrices derived from the tensor decompositions can lead to interpretations for real applications, many nonnegative tensor decompositions have been proposed and developed, and most of them are based on tensor decomposition with  nonnegative constraints.  For Tucker decomposition with nonnegative constraints, that is referred to as Nonnegative Tucker Decomposition (NTD) in \cite{kim2007nonnegative}, aims to solve
%
\begin{equation}\label{ntd}
\begin{split}
&\min \|\mathcal{A}- \mathcal{X} \| \\
&\mbox{s.t.}\quad \mathcal{X}=\mathcal{S}\times_1 \mathbf{P}_{1}\times_2 \mathbf{P}_{2}\times_3\cdots\times_m \mathbf{P}_{m}, \\
 & \quad \quad \mathcal{S}\in \mathbb{R}^{r_1\times\cdots\times r_m}_+, \quad \mathbf{P}_{k}\in \mathbb{R}^{n_k\times r_k}_+,\quad k=1,\cdots m.
 \end{split}
\end{equation}
In \cite{kim2007nonnegative}, Kim and Choi first studied this model and proposed
multiplicative updating algorithms  extended from nonnegative matrix factorization (NMF) to solve it. In \cite{zhou2012fast}, Zhou et al. transformed this problem into a series of NMF problem, and used
MU and HALS algorithms on the unfolding matrices for Tucker decomposition calculation.
Some other constraints like orthogonality on the factor matrices are also considered and studied
by some researchers \cite{xutaoli2017,pan2019orthogonal}.
For instance, in \cite{pan2019orthogonal}, Pan et al. proposed orthogonal nonnegative Tucker decomposition and applied the alternating direction method of multipliers (ADMM), to get clustering informations from the factor matrices and the joint connection weight from the core tensor.


The biggest advantage of NTD model is the core tensor and factor matrices can be interpretable thanks to the requirement of the factorized components. However the approximation $ \mathcal{X}$ is not the best approximation of $\mathcal{A}$ for the given Tucker rank $(r_1,\cdots, r_m)$.  Hence it is
required to find the best low Tucker rank nonnegative approximation for a given nennegative tensor $\mathcal{A}$
with interpretable factor matrices and core tensor. In this paper, we propose the following problem. Given tensor $\mathcal{A}\in \mathbb{R}^{n_1\times\cdots\times n_m}_+$,
\begin{equation}\label{prom01}
\min_{\mathcal{X}\geq 0}
\| {\cal A} - {\cal X} \|_F^2,
\quad\mbox{s.t.}\quad rank_T(\mathcal{X})=(r_1,r_2,\cdots,r_m).
\end{equation}
From $rank_T(\mathcal{X})=(r_1,r_2,\cdots,r_m)$, we can deduce that there exist core tensor ${\cal S} \in \mathbb{R}^{r_1 \times r_2 \times \cdots \times r_m}$ and  orthogonal factor matrices
$\{{\bf P}_{k} :{\bf P}_{k}\in \mathbb{R}^{n_k\times r_k}, {\bf P}^{T}_k {\bf P}_{k}=\mathbf{I}_{r_k}, k=1,\cdots, m\}$, such that
$$
{\cal X} = {\cal S} \times_1 {\bf P}_1 \times_2 {\bf P}_2 \times_3 \cdots \times_m {\bf P}_m.
$$
For $k=1,\cdots,m$, let $\mathbf{X}_k$ be the $k$-th unfolding of tensor $\mathcal{X}$,   defined
as $\mathbf{X}_k\in \mathbb{R}^{n_k\times (n_{k+1}\cdots n_{m}n_{1}\cdots n_{k-1})}$.  From the definition of Tucker decomposition,  we deduce that   $r_k=rank(\mathbf{X}_k)$, and factor matrix $\mathbf{P}_k$ can be obtained by
singular value decomposition on ${\bf X}_k$:
$$
{\bf X}_k = {\bf P}_k {\bf \Sigma}_k {\bf Q}_k^T,
$$
here ${\bf \Sigma}_k$ is a diagonal matrix of size $r_k$-by-$r_k$,
and ${\bf Q}_k$ is $n_k$-by-$r_k$ with orthonormal columns (${\bf Q}_k^T$ is the transpose of
${\bf Q}_k$).

We remark that problem \eqref{prom01} without the nonnegativity constraint on the approximation $\mathcal{X}$ is referred to as
the best low multilinear rank approximation problem, which has been well discussed and used  widely as a tool in dimensionality reduction and signal subspace estimation in recent two decades. The classical methods for the problem are truncated higher-order SVD (HOSVD)\cite{de2000multilinear}  and higher-order orthogonal iteration (HOOI)
\cite{de2000best,kroonenberg2008applied}, proposed based on a higher-order extension of iteration methods for matrices. Without the nonnegative constraint, the solution $\mathcal{X}$ can have negative entries that cannot preserve nonnegative property from the given nonnegative tensor.

Note that in the proposed model \eqref{prom01}, we  require  $\mathcal{X}$ to be nonnegative, while its factorized components $(\mathcal{S},\{\mathbf{P}_{k}\}^m_{k=1})$ are not necessary to be nonnegative.  For example, given hyperspectral image $\mathcal{A}$, $\mathcal{X}$ can be seen as the approximate image to $\mathcal{A}$ but with lower multilinear rank. On one hand, we keep the approximate image $\mathcal{X}$ to be
nonnegative. On the other hand, no constraints are added to the factorized components, so that we may consider  a similar idea that utilized in HOSVD to identify important features in the approximation and these features are ranked based on their importance. Therefore, we can identify the important factorized components for classification purpose, see Section \ref{Exp:features} for an example.

\subsection{Outline and Contributions}

The main aim of this paper is to propose and study low multiliear rank nonnegative tensor approximation for applications of multi-dimensional images. In Section 2, we propose an alternating manifold-projection method
for computing nonnegative low multilinear rank tensor approximation. The projection method is developed by constructing two projections: one is a combination of a projection of low rank matrix manifolds and
the nonnegative projection; the other one is a projection of taking average of tensors.  In Section 3, the convergence of the proposed  method is studied and shown.
In Section 4, experimental results for synthetic data and multi-dimensional images in noisy cases and noise-free cases are presented to demonstrate the performance of the proposed nonnegative low multilinear rank tensor approximation method is better than state-of-the-art NTF methods. Some concluding remarks are given in Section 5.

\section{Nonnegative Low Rank Tensor Approximation}

Let us first start with some tensor operations used throughout this paper.
The inner product of two same-sized tensors $\mathcal{A}$ and $\mathcal{B}$ is defined as
 \[
 \langle\mathcal{A},\mathcal{B}\rangle:=\sum\limits_{i_{1},i_{2},\cdots,i_{m}}
 {\cal A}_{i_{1}i_2\cdots i_{m}}\cdot {\cal B}_{i_{1}i_2\cdots i_{m}}.
 \]
The Frobenius norm of an $m$-dimensional tensor $\mathcal{A}$ is defined as $$\left\|{\mathcal{A}}\right\|_{F}:=\sqrt{\langle{\mathcal{A}},\mathcal{A}\rangle}=\left(\sum\limits_{i_{1},i_{2},\cdots,i_{m}}{\cal A}^2_{i_{1}i_2\cdots i_{m}}\right)^{\frac{1}{2}}.$$
\subsection{The Optimization Model}

We first give the following lemma to demonstrate that
the set of constraints in (\ref{prom01}) is non-empty.

\begin{lemma}
The set of constraints
$\{ {\cal X} \in
\mathbb{R}^{n_1\times n_2\times \cdots\times n_m} \ | \
\rank(\textbf{X}_{k})=r_{k}~ (k=1,...,m), \mathcal{X}\geq 0\}$ in (\ref{prom01})
is non-empty.
\end{lemma}

\begin{proof}
 First, we will prove there always exists a tensor $\mathcal{S}\in \mathbb{R}^{r_1\times\cdots\times r_m}_+$ that has full unfolding matrix rank for each mode.

For any $t\in \mathbb{R}^{r_1r_2\cdots r_m}_+$, let $(\mathbf{S}_k)(t)\in \mathbb{R}^{r_k \times r_1\cdots r_{k-1}r_{k+1}\cdots r_m} $ hold the elements of $t$.  Let $(\mathbf{S}_k)(t)_{r_k }$ be the $r_k\times r_k$ sub matrix of $(S_k)(t)$ and $det((\mathbf{S}_k)(t)_{r_k })$ be its determinant. As we know that $det((\mathbf{S}_k)(t)_{r_k })$ is a polynomial in the entries of $t$, so it either vanishes on a set of zero measure or it is the zero polynomials. We may choose $(\mathbf{S}_k)(t)_{r_k }$ to be the identity matrix, which implies that $det((\mathbf{S}_k)(t)_{r_k })$  is not zero polynomials. This means the Lebesgue measure of the space whose $det((\mathbf{S}_k)(t)_{r_k })=0$ is zero, i.e., the rank of $(\mathbf{S}_k)(t)_{r_k }$ is $r_k$ almost everywhere.

Thus  for $k=1,\cdots, m$, construct $\mathcal{T}_k=\{\mathcal{S}\in \mathbb{R}^{r_1\times \cdots\times r_m}_+| rank(\mathcal{S}_k)=r_k\}$, and $\bar{\mathcal{T}}_k$ be its complement. From the above analysis, we know that the Lebesgue measure of $\bar{T}_k$ is equal to zero. Let $\mathcal{T}=\cap^m_{k=1} \mathcal{T}_k$, then its complement $\bar{T}=\cup^m_{k=1} \bar{\mathcal{T}}_k$,  its Lebesgue measure is the summation of that of $\bar{\mathcal{T}}_k$ through from $k=1$ to $k=m$, equal to zero. It implies that the Lebesque measure of $T$ is equal to 1, i.e., $\mathcal{S}\in \mathbb{R}^{r_1\times\cdots\times r_m}_+$ of unfolding matrix
rank $(r_1,\cdots,r_m)$ exists almost everywhere.

 Suppose $\mathbf{P}_k\in \mathbb{R}^{n_k\times r_k}$, and $\mathbf{P}_k=[\mathbf{I}_{k}| \mathbf{U}_k]$ ,where $\mathbf{I}_k$ is identity matrix of $r_k$, $\mathbf{U}_k\in \mathbb{R}^{r_k\times (n_k-r_k)}$ is a random nonnegative matrix for all $k=1,\cdots, m$. Construct
$$\mathcal{X}=\mathcal{S}\times \mathbf{P}_1\times \cdots\times \mathbf{P}_m,$$
 we get that $\mathcal{X}$ is nonnegative and its multilinear rank is $(r_1,\cdots,r_m)$,
 the set of constraints is non-empty. $\Box$
\end{proof}

From the definition of Tucker decomposition and the property of multilinear rank that $r_k=rank(\mathbf{X}_k)$ for $k=1,\cdots, m$, the mathematical model \eqref{prom01} can be reformulated as the following optimization problem
\begin{equation}\label{prom1}
\min_{\rank(\textbf{X}_{k})=r_{k}, \textbf{X}_{k}\geq 0,\atop
(k=1,...,m)}
\sum_{k=1}^{m}\|\textbf{A}_{k}-\textbf{X}_{k}\|_{F}^2,
\end{equation}
where $\textbf{X}_{k}$ and $\textbf{A}_{k}$ are the $k$-th mode of unfolding matrix of
$\mathcal{X}$ and $\mathcal{A}$, respectively. The sizes of $\textbf{\textbf{A}}_{k}$
and $\textbf{X}_{k}$ are $n_{k}$-by-$N_{k}$ with $N_{k}=\prod_{i\neq k}^{n} n_{i}$.

Note that from \eqref{prom1}, $\{\mathbf{X}_k\}^m_{k=1}$ can be seen as $m$
manifolds of low rank and nonnegative matrices.
Meanwhile, as the Frobenius norm is employed in the objective function, to a certain extent, our model is tolerant to the noise, which is unavoidable in real-world data.
In the next section, an alternating projection on manifolds algorithm will be proposed to solve model (\ref{prom1}).
\subsection{The Proposed Algorithm}

To start showing the proposed algorithm for \eqref{prom1}, we first need to define two projections. Let
\begin{equation}\label{man2}
{\tt M}=\{\mathcal{X}\in \R^{n_{1}\times\cdots\times n_{m}} \ | \
\mathcal{X}_{i_{1}i_{2}\cdots i_{m}}\geq 0 \}
\end{equation}
be the set of nonnegative tensors, then the nonnegative projection that projects a given tensor onto tensor manifold ${\tt M}$ can be expressed as follows:
\begin{align}\label{p2}
\pi({ \mathcal{X}})=\left\{\begin{array}{cc}
                     \mathcal{X}_{i_{1}i_{2}\cdots i_{m}}, &~~ {\rm if} ~~ \mathcal{X}_{i_{1}i_{2}\cdots i_{m}}\geq 0, \\
                     0,     &~~ {\rm if} ~~   \mathcal{X}_{i_{1}i_{2}\cdots i_{m}} < 0.
                   \end{array}
\right.
\end{align}
Let
\begin{equation}\label{man1}
{\tt M}_{k}=\{\mathcal{X}\in \R^{n_{1}\times\cdots\times n_{m}} \ | \ \rank(\textbf{X}_{k})=r_{k} \
 \},~ {k=1,...,m}
\end{equation}
be the set of tensors whose $k$-mode unfolding matrices have fixed rank $r_k$.  By the
Eckart-Young-Mirsky theorem \cite{golub2012matrix}, the $k$-mode projections that project tensor
${\cal X}$ onto
${\tt M}_k$ are presented as follows:
\begin{align}\label{p1}
\pi_{k}({\mathcal{X}})=\textrm{fold}_k \left(\sum_{i=1}^{r_{i}}\sigma_{i}(\textbf{X}_{k}) u_{i}(\textbf{X}_{k}) {v}_{i}(\textbf{X}_{k})^T \right), \quad k=1,...,m,
\end{align}
where
${\bf X}_k$ is the $k$-mode unfolding matrix of ${\cal X}$,
$\sigma_{i}(\textbf{X}_{k})$ is the $i$-th singular values of ${\bf X}_k$,
and their corresponding left and right singular vectors are
$u_{i}(\textbf{X}_{k})$ and $v_{i}(\textbf{X}_{k})$, respectively. ``fold$_k$'' denotes the operator that folds a matrix into a tensor along the $k$-mode.

In model \eqref{prom1}, we note that the multilinear rank of  nonnegative approximation $\mathcal{X}$ is require to be $(r_1,\cdots,r_m)$, which means $\mathcal{X}$ will fall in the intersection of sets $\{\tt{M}_k\}^m_{k=1}$ and nonegative tensor set $\tt{M}$, i.e., $\mathcal{X}\in \bigcap\limits^m_{k=1}(\tt{M}_k\bigcap \tt{M})$. In the following, we define two tensor sets on the product space $\R^{n_{1}\times\cdots\times n_{m}}\times \cdots \times \R^{n_{1}\times\cdots\times n_{m}}$ ($m$ times) and their corresponding projections :

\noindent
$\bullet$
\begin{equation}\label{set1}
{\tt \Omega_1}=\{ (\mathcal{X}_1, {\cal X}_2, \cdots, \mathcal{X}_m):
\mathcal{X}_1 = {\cal X}_2 = \cdots = {\cal X}_m \in {\tt M} \}
\end{equation}
We remark that ${\tt \Omega_1}$ is convex and affine manifold since ${\tt M}$ is a convex set and an affine manifold.  The   projection $\pi_{\tt \Omega_1}$ defined on ${\tt \Omega_1}$ is given by
\begin{eqnarray}\label{pro1}
& & \pi_{\tt \Omega_1}(\mathcal{X}_1,\cdots,{\cal X}_m) \nonumber \\
& = & \left (
\frac{1}{m} \left ( \pi(\mathcal{X}_{1})+\cdots+\pi(\mathcal{X}_{m}) \right ), \cdots,
\frac{1}{m} \left ( \pi(\mathcal{X}_{1})+\cdots+\pi(\mathcal{X}_{m}) \right )
\right ),
\end{eqnarray}
where $\pi$ is defined in \eqref{p2}.

\noindent
$\bullet$
\begin{equation}\label{set2}
\begin{aligned}
{\tt \Omega_2}=
\{ (\mathcal{X}_1, {\cal X}_2, \cdots, \mathcal{X}_m): \mathcal{X}_1 \in {\tt M}_1,
\mathcal{X}_2 \in {\tt M}_2,
\cdots, {\cal X}_m \in {\tt M}_m \}.
\end{aligned}
\end{equation}
For each $i\in \{1,...,m\}$, ${\tt M}_{i}$ is
$C^{\infty}$ manifold (Example 2 in \cite{Lewis2008}),   ${\tt \Omega_2}$ can be hence regarded as a product of $m$
$C^{\infty}$ manifolds, i.e., ${\tt \Omega_2} = {\tt M}_{1} \times {\tt M}_2 \cdots\times {\tt M}_{m} $.  The projection $\pi_{\tt \Omega_2}$ on ${\tt \Omega_2}$ is  given by
\begin{equation}\label{pro2}
\pi_{\tt \Omega_2}(\mathcal{X}) =(\pi_{1}(\mathcal{X}), \cdots, \pi_{m}(\mathcal{X})),
\end{equation}
where $\pi_{k}$ ($k=1,...,m$) are defined in \eqref{p1}.


We alternately project the given $\mathcal{A}$ onto $\tt \Omega_{1}$ and $\tt \Omega_{2}$ by the  projections
$\pi_{\tt \Omega_1}(\mathcal{X})$ and $\pi_{\tt \Omega_1}(\mathcal{X})$ until it is convergent, and refer the algorithm to as alternating projections algorithm for nonnegative low rank tensor approximation (NLRT) problem. The proposed algorithm is summarized in Algorithm \ref{ag1}.
Note that
the dominant overall computational cost of Algorithm \ref{ag1} can be expressed
as the SVDs of {$m$}
 unfolding matrices with sizes $n_{k}$ by $N_{k}=\Pi_{i\neq k}^{n} n_{j}$, respectively, which leads  to a total of {$O((\Pi_{j=1}^{m} n_{j}) \sum_{i=1}^{m}r_{i})$}
flops.

\begin{algorithm}[h]
\caption{Alternating Projections Algorithm for Nonnegative Low Rank Tensor Approximation (NLRT)} \label{ag1}
\textbf{Input:}
Given a nonnegative tensor $\mathcal{A} \in \mathbb{R}^{n_{1}\times\cdots\times n_{m}}$, this algorithm computes a Tucker rank $(r_1,r_2,...,r_m)$ nonnegative tensor close to ${\cal A}$ with respect to
(\ref{prom1}). \\
~~1: Initialize $\mathcal{Z}^{(0)}_1 = ... = {\cal Z}_m^{(0)} = \mathcal{A}$ and
${\cal Z}^{(0)} = (\mathcal{Z}_{1}^{(0)},\mathcal{Z}_2^{(0)},...,\mathcal{Z}_m^{(0)})$ \\
~~2: \textbf{for} $s=1,2,...$ ($s$ is the iteration number) \\
~~3: \quad $(\mathcal{Y}_1^{(s)},\mathcal{Y}_2^{(s)},...,\mathcal{Y}_m^{(s)})
=\pi_{\tt \Omega_1}(
\mathcal{Z}^{(s-1)}_1,
\mathcal{Z}^{(s-1)}_2, \cdots,
\mathcal{Z}^{(s-1)}_m)$; \\
~~4: \quad $(\mathcal{Z}_{1}^{(s)},\mathcal{Z}_2^{(s)}, \cdots ,\mathcal{Z}_m^{(s)})
= \pi_{\tt \Omega_2}(
\mathcal{Y}_1^{(s)},\mathcal{Y}_2^{(s)},...,\mathcal{Y}_m^{(s)})$; \\
~~5: \textbf{end}\\
\textbf{Output:} $
{\cal Z}^{(s)} = (\mathcal{Z}_{1}^{(s)},\mathcal{Z}_2^{(s)}, \cdots ,\mathcal{Z}_m^{(s)})$
when the stopping criterion is satisfied.
\end{algorithm}

\section{The Convergence Analysis}

The framework of this algorithm is the same as the convex case for finding a point in the intersection of several closed sets, while the projection sets here are two product manifolds.
In \cite{Lewis2008}, Lewis and Malick proved that a sequence
of alternating projections converges locally linearly if the two projected sets are $C^2$-manifolds
intersecting transversally. Lewis et al. \cite{Lewis2009} proved local linear convergence when two
projected sets intersecting nontangentially in the sense of linear regularity, and
one of the sets is super regular. Later Bauschke et al. \cite{Bauschke20131,Bauschke20132} investigated the case of nontangential intersection further and proved linear convergence under weaker regularity and transversality hypotheses. In \cite{noll2016}, Noll and Rondepierre generalized the existing results by studying
the intersection condition of the two projected sets. They esatablished
local convergence of alternating projections between subanalytic sets under a mild regularity hypothesis on one of the sets.
Here we analyze the convergence of the alternating projections algorithm by using
the results in \cite{noll2016}.

We remark that the sets ${\tt \Omega_1}$ and ${\tt \Omega_2}$ given in \eqref{set1} and \eqref{set2}
respectively are two $C^{\infty}$ smooth manifolds which are not closed. The convergence cannot be derived directly by applying the convergence results
of alternating projections between two closed subanalytic sets.
By using the results in variational analysis and differential geometry, the main
convergence results are shown in the following theorem.

\begin{theorem} \label{mainthm}
Let ${\tt M}_{i},i=1,..,m$ and ${\tt M}$ be the manifolds given
in \eqref{man1} and \eqref{man2}
respectively. Let ${\cal M} \in {\tt M}_{1}\cap \cdots \cap {\tt M}_{m}\cap {\tt M}\neq \emptyset$. Then
there exists a neighborhood ${\tt U}$ of ${\cal M}$ such that whenever a sequence ${\cal Z}^{(s)}$ derived by
Algorithm 1 enters ${\tt U}$, then it converges to some ${\cal Z}^*\in {\tt M}_{1}\cap \cdots \cap {\tt M}_{m}\cap {\tt M}$ with rate
$\| {\cal Z}^{(s)}-{\cal Z}^*\|_{F}=O(s^{-\delta})$ for some $\delta\in(0,+\infty)$.
\end{theorem}

In order to show Theorem \ref{mainthm}, it is necessary to study
H$\ddot{o}$lder regularity and separable intersection. For detailed discussion, we refer to
Noll and Rondepierre \cite{noll2016}.

\begin{definition} \label{holdercon}\cite{noll2016}
Let ${\tt A}$ and ${\tt B}$ be two sets of points in a Hilbert space
equipped with the inner product $\langle \cdot, \cdot \rangle$ and the norm $\| \cdot \|$. Denote $p_{{\tt A}}(x)=\{a\in {\tt A}:\|x-a\|=d_{{\tt A}}(x)\}$, where   $d_{{\tt A}}(x)=\min\{\|x-a\|:a\in {\tt A}\}.$  $p_{{\tt B}}(x)$ can be similarly defined relate to set ${\tt B}$.
Let $\sigma \in [0,1)$. The set ${\tt B}$ is $\sigma$-H$\ddot{o}$lder regular
with respect to ${\tt A}$ at $x^* \in {\tt A} \cap {\tt B}$ if there exists a neighborhood
${\tt U}$ of $x^*$ and a constant
$c > 0$ such that for every $y^+ \in {\tt A} \cap {\tt U}$ and
every $x^+ \in p_{\tt B}(y^+) \cap {\tt U}$, one has
$$
Ball( y^+, (1+c)r ) \cap \{ x \ | \  y^+ \in p_{\tt A}(x),
\langle y^+ - x^+, x - x^+ \rangle > \sqrt{c} r^{\sigma+1}
\| x - x^+ \| \} \cap {\tt B} = \emptyset,
$$
where $r = \| y^+ - x^+ \|$. Note that
$p_{\tt B}(y^+)$ is the projection of $y^+$ onto ${\tt B}$
and
$p_{\tt A}(x)$ is the projection of $x$ onto ${\tt A}$,
with respect to the norm.
We say that ${\tt B}$ is H$\ddot{o}$lder
regular with respect to ${\tt A}$ if it is
$\sigma$-H$\ddot{o}$lder regular with respect to ${\tt A}$ for every $\sigma \in
[0,1)$.
\end{definition}

H$\ddot{o}$lder regularity is mild compared with some other
regularity concepts such as the prox-regularity \cite{rockafellar2009variational}, Clarke regularity \cite{clarke1990regularity}
and super-regularity \cite{lewis2009local}.

\begin{definition} \cite{noll2016} \label{separa}
Let ${\tt A}$ and ${\tt B}$ be two sets of points in a Hibert space
equipped with the inner product $\langle \cdot, \cdot \rangle$ and the norm $\| \cdot \|$.
We say ${\tt B}$ intersects separably ${\tt A}$ at $x^{*}  \in  {\tt A} \cap {\tt B}$
with exponent $\omega \in [0,2)$ and constant
$\gamma>0$ if there exist a neighborhood ${\tt U}$ of $x^*$ such that for every building block
$z \rightarrow y^{+} \rightarrow z^{+}$ in ${\tt U}$, the condition
\begin{equation}
\langle z-y^{+} , z^{+} - y^{+} \rangle
\leq (1-\gamma \| z^{+}- y^{+} \|^{\omega})
\| y - z^{+}\| \| z^{+}- y^{+}\|
\end{equation}
holds, i.e., it is equivalent to
$$
\frac{1-\cos \alpha}{\| y^{+}- z^{+}\|^{\omega}}
\geq \gamma,
$$
where $y^+$ is a projection point of $z$ onto ${\tt A}$,
$z^+$  is a projection point of $y^+$ onto ${\tt B}$, and
$\alpha$ is the angle between $z-y^{+}$
and $z^{+}- y^{+}$.
\end{definition}

This separable intersection definition is a new geometric concept which generalized the transversal intersection \cite{Lewis2008}, the linear regular intersection \cite{Lewis2009}, and the intrinsic transversality intersection \cite{drusvytskiy2014}.
It has been shown that the definitions of these three kinds of intersections
imply $\omega=0$ in the separable intersection.

The following results are needed to prove our main results.

\begin{theorem}[Theorem 1 and Corollary 4 in \cite{noll2016}]
Suppose ${\tt B}$ intersects ${\tt A}$ separably at
$x\in {\tt A}\cap {\tt B}$ with exponent $\omega\in [0,2)$ and constant $\gamma$ and is $\omega/2$-H$\ddot{o}$lder regular at $x$ with respect to ${\tt A}$ and constant $c<\frac{\gamma}{2}.$ Then there exist a neighborhood ${\tt U}$ of $x$ such
that every sequence of alternating projections between ${\tt A}$ and ${\tt B}$ which enters ${\tt U}$ converges to a point $c^{*}\in {\tt A}\cap {\tt B}$ with convergence rate as $b_{k}-c^{*}=O(k^{-\frac{2-\omega}{2\omega}})$ and $a_{k}-c^{*}=O(k^{-\frac{2-\omega}{2\omega}}).$
\end{theorem}


\vspace{2mm}
\noindent
{\it Proof of Theorem 1}.  Let ${\tt \Omega_1}$ and ${\tt \Omega_2}$ be given as \eqref{set1} and \eqref{set2}.  It is clear that finding a point in
${\tt M}_{1}\cap \cdots \cap {\tt M}_{m}\cap {\tt M}$
is equivalent to finding a point in the intersection of ${\tt \Omega_1}$ and ${\tt \Omega_2}$.

The first task is to show that
${\tt \Omega_2}$ intersects separably ${\tt \Omega_1}$ at ${\cal X}^* \in
{\tt \Omega_1} \cap {\tt \Omega_2}$ with exponent $\omega \in (0,2)$.
Define
$f: {\tt \Omega_1} \rightarrow \R$ as
\begin{equation}\label{f1}
f( {\cal X} ) =
\delta_{\tt \Omega_1}( {\cal X} ) +
\frac{1}{2}d_{\tt \Omega_{2}}^2( {\cal X} ), \quad {\cal X} = ({\cal X}_1,{\cal X}_2,...,{\cal X}_m) \in
{\tt \Omega_1},
\end{equation}
with
\begin{equation*}
\delta_{\tt \Omega_1}({\cal X})=\left\{
\begin{array}{cl}
0~& ~ {\rm if} ~ {\cal X} \in {\tt \Omega_1}, \\
+\infty  ~ &~ {\rm otherwise}
\end{array}
\right.
\end{equation*}
and
$$
d_{\tt \Omega_{2}}( {\cal X} ) =
\min\{\|
({\cal X} - {\cal W} \|_{F} : {\cal W} \in {\tt \Omega_{2}} \}.
$$
It follows the deﬁnition of $f( {\cal X} )$ that
$f( {\cal X}^* ) = 0$ and ${\cal X}^*$ is a critical point of $f$.

Recall that ${\tt \Omega_{1}}$ and ${\tt \Omega_{2}}$ are two $C^{\infty}$ manifolds.
Then $f$ is locally Lipschitz continuous,
i.e., for each ${\cal X}
\in {\tt \Omega_{1}}$,
there is an $r>0$ such that
$f$ is Lipschitz continuous on the open ball of center
${\cal X}$ with radius $r$.
Assume that $({\tt V},\psi)$ is a local smooth chart of ${\tt \Omega_{1}}$ around ${\cal X}^*$ with
bounded ${\tt V}$.
Therefore, $f({\tt V})$ is bounded by the fact that $f$ is local Lipschitz continuous.
According to the definition of semi-algebraic function \cite{li2016douglas},
we can deduce that $f \circ \psi^{-1}$ is also semi-algebraic.
Then the Kurdyka-{\L}ojasiweicz inequality \cite{Attouch2010} for $f \circ \psi^{-1}$ holds for
$\bar{\cal W} :=\psi( {\cal X}^{*} )$.
It implies that there exist $\eta \in (0,\infty)$ and a concave function $\tau:[0,\eta]$ such that
\begin{enumerate}
\item[(i)] $\tau(0)=0$;
\item[(ii)] $\tau$ is $C^{1}$;
\item[(iii)] $\tau'>0$ on $(0,\eta)$;
\item[(iv)] for all ${\cal W} \in \psi({\tt V})= {\tt U}$
with $
f \circ\psi^{-1}( \bar{\cal W} )
< f\circ\psi^{-1}( {\cal W} )
< f\circ\psi( \bar{\cal W}) )+\eta$,
we have
$$
\tau' ( f\circ\psi^{-1}( {\cal W} ) -f\circ\psi^{-1}( \bar{\cal W}) ) \
\textrm{dist}(0, \partial (f\circ\psi^{-1})( {\cal W} ) \geq 1.
$$
\end{enumerate}
Moreover, $\tau$ is analytic on ${\tt V}$, thus $D(\psi)$ is continuous on ${\tt V}$, where $D$
is the differential operator. For every compact subset ${\tt K}$ in ${\tt V}$, there exists
$C_{\tt K}:=\sup_{ {\cal W} \in {\tt K}} \| D ( \psi( {\cal W} ) ) \|$,
where $\|\cdot\|$ denotes the operator norm.
Suppose that ${\tt V}^{'}$ is an open set containing
${\cal X}^{*}$ in ${\tt V}$ such that ${\tt K}=cl({\tt V}^{'}) \subset int ({\tt V})$ is compact
($cl({\tt V}^{'})$ denotes the closure of ${\tt V}^{'}$ and $int({\tt V})$ denotes the $interior$ of
${\tt V}$).
Then, for every ${\cal X} \in {\tt V}^{'}$ with
$f({\cal X}^{*}) < f({\cal X}) < f({\cal X}^{*}) +\eta$,
we have
\begin{equation}\label{Thm1_eq1}
C_{\tt K}\tau'( f({\cal X}) - f({\cal X}^{*}) ) \
\textrm{dist}(0,\hat{\partial} ( f({\cal X}) ) \geq 1,
\end{equation}
where $\hat{\partial} f({\cal X})$ is the Fr\'{e}chet subdifferential of $f$.
We see that the Kurdyka-{\L}ojasiweicz inequality is satisfied for $f$ given in \eqref{f1}.

 Here we construct a function $\tau=t^{1-\theta}$ which satisfies  (i)-(iv). Because $f({\cal X}^{*})=0$, (\ref{Thm1_eq1}) becomes
\begin{equation*}
C_{\tt K}\tau'( f({\cal X})) \
\textrm{dist}(0,\hat{\partial} ( f({\cal X}) ) \geq 1.
\end{equation*}
Since $\tau'(t)=(1-\theta)t^{-\theta}$,
there always exists a neighborhood ${\tt U}$ of ${\cal X}^{*}\in {\tt \Omega_{1}} \cap {\tt \Omega_{2}}$ such that
$
C_{\tt K}(1-\theta)|f( {\cal X})|^{-\theta} \|g\|_{F}\geq 1
$, i.e.,
\begin{align}\label{nnew2}
|f( {\cal X})|^{-\theta} \|g\|_{F}\geq c, \quad \text{with} \ c=\frac{1}{C_{\tt K}(1-\theta)},
\end{align}
for all ${\cal X}\in {\tt \Omega_{1}} \cap {\tt U}$ and every
$g\in \hat{\partial} f( {\cal X} )$.

In Algorithm 1,  we construct the following sequences according to Definition \ref{separa}:
$$
{\cal Z} \rightarrow {\cal Y}^+ \rightarrow {\cal Z}^+.
$$
Here ${\cal Y}^+$ is the projection $\pi_{\tt \Omega_1}({\cal Z})$
and
${\cal Z}^+$ is the projection $\pi_{\tt \Omega_2}({\cal Y}^+)$  with  $\pi_{\tt \Omega_1}({\cdot})$ and $\pi_{\tt \Omega_2}({\cdot})$ being defined as \eqref{pro1} and \eqref{pro2}, respectively.  Suppose ${\cal Z}$ and ${\cal Z}^+$ are in ${\tt U}$,  ${\cal Y}^+ \in {\tt U}  \cap {\tt \Omega_{1}}$, we get the proximal normal cone
to ${\tt \Omega_{1}}$ at ${\cal Y}^+$:
$$
{\tt N}_{\tt \Omega_{1}}^{p}({\cal Y}^+) =
\{ \lambda {\cal V} : \lambda \geq 0, {\cal Y}^+ \in \pi_{\tt \Omega_{1}}( {\cal Y}^+ + {\cal V} ) \}.
$$
According to the definition of Fr\'{e}chet subdifferential,
${\cal G} \in \hat{\partial} f( {\cal Y}^+ )$
if and only if
${\cal G}  = {\cal V} + {\cal Y}^+ - {\cal Z}^+$ for every ${\cal V} \in
{\tt N}_{\tt \Omega_{1}}^{p}({\cal Y}^+)$ of the form
${\cal V} = \lambda ({\cal Z} - {\cal Y}^+)$.

Note that  ${\cal Y}^+ \in \pi_{\tt \Omega_{1}}( {\cal Z} )$, from (\ref{f1}), we have
$f ( {\cal Y}^+ ) = \frac{1}{2}d_{\tt \Omega_{2}}^{2}( {\cal Y}^+)$.  Substitute $f ( {\cal Y}^+ )$  into \eqref{nnew2} gives
$$
2^{\theta}d_{\tt \Omega_{2}}( {\cal Y}^+ )^{-2\theta}
\| \lambda( {\cal Z} - {\cal Y}^+ ) +
( {\cal Y}^+ - {\cal Z}^+ ) \|_{F}\geq c>0,
$$
for every $\lambda\geq 0$.
It follows that
\begin{align}\label{nnew3}
d_{\tt \Omega_{2}}( {\cal Y}^+ )^{-2\theta}
\min_{\lambda\geq 0} \| \lambda( {\cal Z} - {\cal Y}^+ ) +
( {\cal Y}^+ - {\cal Z}^+ ) \|_{F}
\geq 2^{-\theta} c.
\end{align}
Let the angle $\alpha$ be the angle between the iterations, which can be defined as the angle between
${\cal Z} - {\cal Y}^+$ and ${\cal Z}^+ - {\cal Y}^+$.

Let us consider two cases.

(i) When $\alpha\leq \pi/2$,
$$
\min_{\lambda\geq 0} \|
\lambda( {\cal Z} - {\cal Y}^+ ) +
( {\cal Y}^+ - {\cal Z}^+ ) \|_{F} = \| {\cal Y}^+ - {\cal Z}^+ \|_F \sin \alpha,
$$
Substitute it into (\ref{nnew3}), then
$$
\frac{\sin \alpha}{d_{\tt \Omega_{2}}( {\cal Y}^+ )^{2\theta-1}}\geq 2^{-\theta}c.
$$
Note that $1-\cos\alpha\geq \frac{1}{2}\sin^{2}\alpha$, we have
\begin{equation}\label{eq2}
\frac{1-\cos \alpha}{d_{\tt \Omega_{2}}( {\cal Y}^+ )^{4\theta-2}}\geq 2^{-2\theta-1}c^{2}.
\end{equation}
when the numerator tends to $0$,
the denominator has to go to zero, which implies that $4\theta-2>0$, i.e., $\theta>\frac{1}{2}$. Therefore, we get ${\tt \Omega_{2}}$ intersects ${\tt \Omega_{1}}$ separably with exponent
$\omega=4\theta-2 \in (0,2)$, the corresponding
constant can be set as $c^{'}=2^{-2\theta-1}c^{2}$.

(ii)
When $\alpha>\pi/2$, we have $\cos\alpha<0$, i.e.,
$1-\cos \alpha\geq1$.
The infimum in \eqref{nnew3} is attained at $\lambda=0$. \eqref{nnew3} becomes $d_{\tt \Omega_{2}}( {\cal Y}^+ )^{1-2\theta}\geq 2^{-\theta}c$. Therefore,
$$
d_{\tt \Omega_{2}}( {\cal Y}^+ )^{2-4\theta}\geq 2^{-2\theta}c^{2}>2^{-2\theta-1}c^{2}.
$$
\eqref{eq2} is also satisfied.
According to Definition \ref{separa},   ${\tt \Omega_{2}}$ intersects ${\tt \Omega_{1}}$ separably.

On the other hand,
${\tt \Omega_{1}}$ intersects ${\tt \Omega_{2}}$ separably can be proved by using the
similar argument.

Moreover, ${\tt \Omega_{1}}$ is prox-regularity at ${\cal X}^*$ with arbitrary $\sigma\in [0,1)$, hence ${\tt \Omega_{1}}$ is H$\ddot{o}$lder regular with respect to ${\tt \Omega_{2}}$ at ${\cal X}^*$ .  It follows from Theorem 2
that
there exists a neighborhood ${\tt U}$ of ${\cal X}^*$ such that every sequence of alternating projections that enters ${\tt U}$ converges to ${\cal Z}^* \in {\tt \Omega_1} \cap {\tt \Omega_2}$.
The convergence rate is $\| {\cal Z}^{(s)} - {\cal Z}^* \|_{F}=O(s^{-\frac{2-\omega}{2\omega}})$
and $\| {\cal Y}^{(s)} - {\cal Z}^* \|_{F} = O(s^{-\frac{2-\omega}{2\omega}})$ with
$\omega \in(0,2)$. The result follows.

In the next section, we test our method and nonnegative tensor decomposition methods on the synthetic data and real-world data, and show the performance of the proposed alternating projections
method is better than the others.

\section{Experimental Results}\label{Experiment}

\subsection{Compared methods}

The state-of-the-art methods for nonnegative tensor decompositions   are used as follows.
\begin{itemize}
\item Nonnegative Tucker decomposition (NTD):\\
NTD-HALS: An HALS algorithm \cite{zhou2012fast}\\
NTD-MU: A multiple updating algorithm \cite{zhou2012fast}\\
NTD-BCD: A block coordinate descent method \cite{xu2013block}\\
NTD-APG: An accelerated proximal gradient algorithm \cite{zhou2012fast}
\end{itemize}

We also compare the proposed model with well known nonnegative  CANDECOMP/PARAFAC decomposition (NCPD), that is, given a tensor $\mathcal{A}\in \mathbb{R}^{n_1\times n_2\times \cdots\times n_m}_+$,
\begin{equation}\label{CP}
\begin{split}
&\min\|\mathcal{A}-\sum^Z_{z=1}\lambda_z \mathbf{a}^{z,1} \otimes \mathbf{a}^{z,2} \otimes
\cdots \mathbf{a}^{z,m}\|, \\
&\mbox{s.t.}\quad \mathbf{A}^t= \left(
              \begin{array}{ccc}
                 \mathbf{a}^{1,t}  &\cdots & \mathbf{a}^{Z,t} \\
                \end{array}
            \right)\geq 0,\quad \lambda=\left(
              \begin{array}{ccc}
                 \lambda_1  &\cdots & \lambda_Z \\
                \end{array}
            \right)\geq 0,\quad t=1,\cdots,m.
\end{split}
\end{equation}
The state-of-the-art methods for NCPD model are presented as follows.
\begin{itemize}
\item Nonnegative CP decomposition (NCPD):\\
NCPD-HALS: A hierarchical ALS algorithm \cite{cichocki2007hierarchical,cichocki2009fast}\\
NCPD-MU: A fixed point (FP) algorithm with multiplicative updating \cite{welling2001positive}\\
NCPD-BCD: A block coordinate descent (BCD) method \cite{xu2013block}\\
NCPD-APG: An accelerated proximal gradient method \cite{zhang2016fast}\\
NCPD-CDTF: A block coordinate descent method \cite{shin2016fully}\\
NCPD-SaCD: A saturating coordinate descent method with Lipschitz continuity-based element importance updating rule \cite{balasubramaniam2020efficient}
\end{itemize}

In the following, we list the computational cost of these methods in Table \ref{comp_cost}.
The cost of the proposed NLRT method per iteration
is about the same as that of NTD-type methods.
As they involve the calculation of nonnegative vectors only, the cost of NCP-type methods
per iteration is smaller than that of the proposed NLRT method.

\begin{table}[htbp]
\renewcommand\arraystretch{0.8}
\setlength{\tabcolsep}{3pt}
\centering\scriptsize
\caption{The computational cost.}
\begin{tabular}{ccp{6cm}}
\toprule
        Method      &   Complexity & Details of most expensive compuations\\
    \midrule
    NCPD-mu     &   $O(mr\Pi_{j=1}^{m} n_{j})$   & Khatri-Rao product and unfolding matrices times Khatri-Rao product.  \\\midrule
    NCPD-HALS   &   $O(mr\Pi_{j=1}^{m} n_{j})$& Khatri-Rao product and unfolding matrices times Khatri-Rao product.\\\midrule
    NCPD-BCD    &   $O(mr\Pi^m_{j=1} n_{j})$ &  Khatri-Rao product and unfolding matrices times Khatri-Rao product.\\\midrule
    NCPD-APG    &   $O(mr\Pi_{j=1}^{m} n_{j})$   & Khatri-Rao product and  unfolding matrices times Khatri-Rao product.  \\\midrule
    NCPD-CDTF   &   $O(m^2r\Pi_{j=1}^{m} n_{j})$   & Khatri-Rao product of rank one components and vectors times Khatri-Rao product.  \\\midrule
    NCPD-SaCD   &   $O(mr\Pi_{j=1}^{m}n_{j})$   & Khatri-Rao product and  unfolding matrices times Khatri-Rao product.  \\\midrule\midrule
    NTD-MU      &   $O(\sum^m_{i=1}\Pi^m_{j\neq i} n_{j}r^2_i)$ & MU on unfolding matrices $\{\mathbf{A}_{i}\}^m_{k=1}$.\\\midrule
    NTD-HALS    &   $O(\sum^m_{i=1}\Pi^m_{j\neq i} n_{j}r_i)$ & HALS on unfolding matrices $\{\mathbf{A}_{i}\}^m_{k=1}$.\\\midrule
    NTD-BCD     &   $O(\sum^m_{i=1}\Pi^m_{j\neq i} n_{j}r_i(r_i+n_i))$ &The tensor-matrix multiplication  and the matrix multiplication between the $i$-th unfolding matrix of $\mathcal{G}\times_{j=1,j\neq i} \mathbf{U}^{(j)}$ and its transpose.\\\midrule
    NTD-APG  & $O(\sum^m_{i=1}\Pi^m_{j\neq i} n_{j}r^2_i)$    & The tensor-matrix multiplications among a) the $i$-th factor matrix  b) the transpose of the $i$-th unfolding matrix of $\mathcal{G}\times_{j=1,j\neq i\mathbf{U}^{(j)}}$ and c) the $i$-th unfolding matrix of  $\mathcal{G}\times_{j=1,j\neq i\mathbf{U}^{(j)}}$.\\\midrule
    \midrule
    NLRT        &   $O((\Pi_{j=1}^{m} n_{j}) \sum_{i=1}^{m}r_{i})$   &SVDs of unfolding matrices  $\{\mathbf{A}_{i}\}^m_{k=1}$. \\

\bottomrule

\end{tabular}%
\label{comp_cost}
\end{table}

The stopping criterion of the proposed method and other comparison methods is that
the relative difference between successive iterates is smaller than $10^{-5}$.
All the experiments are conducted on Intel(R) Core(TM) i9-9900K CPU@3.60GHz with 32GB of RAM using Matlab.   Throughout this section, we mainly test the low-rank approximation ability of our method and nonnegative tensor decomposition methods with given rank. That is the CP rank and the multilinear rank are manually prescribed. As for real-world applications, we suggest two adaptive rank adjusting strategies proposed in \cite{xu2015parallel}. The basic idea is to use a large (or a small) value of the rank as the initial guess and adaptively decrease (or increase) the rank based on the QR decomposition of unfolding matrices as the algorithm iterates. The effectiveness of those strategies have been revealed in \cite{xu2015parallel}.

\subsection{Synthetic Datasets}\label{sec:synth}

We first test different methods on synthetic datasets.
We generate two kinds of synthetic data as follows:

\begin{itemize}
\item Case 1 (Noisy nonnegative low-rank tensor):
We generate low rank nonnegative tensors by two steps.
First, a core tensor of the size $r_1\times r_2\times \cdots\times r_m$ (i.e.,
multilinear rank is $(r_1,r_2,\cdots,r_m)$) and $m$
factor matrices of sizes $n_i \times r_i$ ($i = 1,2,\cdots,m$) are generated with entries uniformly distributed in $[0,1]$.
Second, these factor matrices are multiplied to the core tensor via the tensor-matrix product
to generate the low rank nonnegative tensors of size $n_1\times n_2\times \cdots\times n_m$, and each entry  is element-wisely divided by the maximal value, being in the interval of $[0,1]$.
Finally, we add Gaussian noise to generate noisy tensors with different signal-to-noise ratios (SNR)\footnote{To avoid making the entries negative, we first simulate a noise with standard normal distribution, and then set the negative noisy value to be 0. The SNR in dB is defined as $\text{SNR}_\text{dB} = 20\log_{10}\frac{\|\mathcal{X}_\text{groundtruth}\|_F}{\|\text{Noise}\|_F}$.}.

\item Case 2 (Nonnegative random tensor):
We randomly generate
nonnegative tensors of the size $n_1\times n_2\times\cdots\times n_m$ where their entries follow a uniform distribution in between 0 and 1.
The tensor data is fixed once generated and the low rank minimizer is unknown in this setting.
For CP decomposition methods, the CP rank is set to be $r$.
For Tucker decomposition methods, the multilinear rank is set to be
$[r,r,\cdots,r]$.
\end{itemize}

  It is not straightforwardly easy
to make the comparison between the NCPD methods with low multilinear rank based methods fairly, owing to different definitions of the rank.
For NCPD methods, determining the CP rank of a given tensor is NP-hard \cite{kolda2009tensor}. Fortunately, we have that,
given the multilinear rank ($r_1,r_2,\cdots,r_m$) of a tensor,
its CP rank cannot be larger than $\prod_{k=1}^{m} r_k$.
Therefore, in Case 1, we select the CP rank in
the NCPD methods from a set with three candidates, i.e.,
$\{\prod_{k=1}^{m}r_k, \sum_{k=1}^{m}r_k, \max_i{r_i} \}$. Then, we report the best relative approximation error in the NCPD methods.
We believe this makes the comparison with the NCPD methods possible and fair to a certain extent in Case 1.  In Case 2, we set the CP rank as $r$ for NCPD methods when the multilinear rank is $[r,r,\cdots,r]$. In this situation, the results by NCPD methods only reflect the representation ability of these NCPD methods.

We report the relative approximation error\footnote{Defined as
$
\frac{\|\mathcal{X}_\text{estimated}-\mathcal{X}_\text{groundtruth}\|_F}{\|\mathcal{X}_\text{groundtruth}\|_F}.
$}
to quantitatively measure the approximation quality. The ground truth tensor is the generated tensor without noise.
The relative approximation errors of the results by different methods in Case 1 are reported in
Table \ref{tab-test1}. The reported entries of all the comparison methods
in the table are the average values together with the standard deviations of ten trails with different random initial guesses in CP decomposition vectors and Tucker decomposition matrices. However,
the results of the proposed NLRT method are deterministic when the input nonnegative tensor is fixed.
We can see from Table \ref{tab-test1} that the proposed NLRT method achieves the best
performance and it is also quite robust to different noise levels.

In table \ref{tab-test1-time}, we report the average running time of each method. For the tensors with the same size, NCPD methods and NTD methods respectively need the same computation time for different noise levels.
The running time of our NLRT becomes less when the SNR value is larger. This indicates that our method could converge faster with less noise.
Meanwhile, we can see that as the number of total elements in the tensor grows from $10^6$ ($100\times100\times100$) to $2.3\times10^7$ ($30\times30\times30\times30\times30$), the running time of all the methods increases rapidly.
Since that our method involves computations of SVD, whose computation complexity grows cubic to the dimension, our superior of efficiency is obvious for smaller data.
\begin{table}[!t]
\renewcommand\arraystretch{0.8}
\setlength{\tabcolsep}{2pt}
\centering\scriptsize
\caption{  The mean values (and standard deviations) of relative approximation errors of the results by different methods in Case 1. The \textbf{best} values are highlighted in bold. (The mean values and standard deviations are shown in percentages.)}\vspace{-1.5mm}
\begin{tabular}{cc ccccccccc ccccc cccc}\toprule
 \multicolumn{8}{c}{Tensor size: $100\times 100\times 100$}&\multicolumn{8}{c}{Multilinear rank: $[5,5,5]$}\\\midrule
 SNR &\multirow{2}{*}{Noisy}&& \multicolumn{6}{l}{NCPD-}    
 && \multicolumn{4}{l}{NTD-} 
 &&   \multirow{2}{*}{NLRT}   \\
 (dB)&                      && MU       & HALS  & APG   & BCD   & CDTF  & SaCD  && MU   & HALS  & APG   & BCD     &&                           \\\midrule
 \multirow{2}{*}{30} &\multirow{2}{*}{  3.16} &&  2.86&  2.78&  2.74&  2.74&  2.75&  2.95&&  2.84&  2.75&  2.73&  2.75&&\multirow{2}{*}{\bf   2.73}\\
                         &&&  (0.01)&  (0.01)&  (0.00)&  (0.00)&  (0.00)&  (0.11)&&  (0.05)&  (0.03)&  (0.00)&  (0.01)&&                            \\\midrule
 \multirow{2}{*}{40} &\multirow{2}{*}{  1.00} &&  1.21&  1.01&  0.87&  0.87&  0.87&  1.31&&  1.11&  1.00&  0.88&  0.95&&\multirow{2}{*}{\bf   0.86}\\
                         &&&  (0.02)&  (0.01)&  (0.00)&  (0.00)&  (0.00)&  (0.14)&&  (0.15)&  (0.12)&  (0.01)&  (0.05)&&                            \\\midrule
 \multirow{2}{*}{50} &\multirow{2}{*}{  0.32} &&  0.91&  0.59&  0.28&  0.28&  0.28&  0.97&&  0.67&  0.50&  0.33&  0.51&&\multirow{2}{*}{\bf   0.27}\\
                         &&&  (0.02)&  (0.03)&  (0.00)&  (0.00)&  (0.00)&  (0.21)&&  (0.19)&  (0.16)&  (0.01)&  (0.07)&&    \\\midrule
 SNR &\multirow{2}{*}{Noisy}&& \multicolumn{6}{l}{NCPD-}    
 && \multicolumn{4}{l}{NTD-} 
 &&   \multirow{2}{*}{NLRT}   \\
 (dB)&                      && MU       & HALS  & APG   & BCD   & CDTF  & SaCD  && MU   & HALS  & APG   & BCD     &&                           \\\midrule
 \multirow{2}{*}{30} &\multirow{2}{*}{  3.16} &&  2.94&  2.77&  2.74&  2.75&  2.75&  2.99&&  2.68&  2.67&  2.67&  2.67&&\multirow{2}{*}{\bf   2.66}\\
                         &&&  (0.01)&  (0.00)&  (0.00)&  (0.01)&  (0.01)&  (0.16)&&  (0.01)&  (0.00)&  (0.00)&  (0.00)&&                            \\\midrule
 \multirow{2}{*}{40} &\multirow{2}{*}{  1.00} &&  1.40&  0.96&  0.87&  0.88&  0.88&  1.41&&  0.91&  0.88&  0.86&  0.85&&\multirow{2}{*}{\bf   0.84}\\
                         &&&  (0.03)&  (0.02)&  (0.00)&  (0.02)&  (0.01)&  (0.13)&&  (0.04)&  (0.02)&  (0.01)&  (0.01)&&                            \\\midrule
 \multirow{2}{*}{50} &\multirow{2}{*}{  0.32} &&  1.14&  0.52&  0.29&  0.34&  0.32&  1.22&&  0.41&  0.35&  0.31&  0.31&&\multirow{2}{*}{\bf   0.27}\\
                         &&&  (0.03)&  (0.03)&  (0.01)&  (0.09)&  (0.04)&  (0.29)&&  (0.05)&  (0.03)&  (0.02)&  (0.03)&&    \\\toprule
\multicolumn{8}{c}{Tensor size: $30\times 30\times 30\times 30\times 30$}&\multicolumn{8}{c}{Multilinear rank: $[2,2,2,2,2]$}\\\midrule
 SNR &\multirow{2}{*}{Noisy}&& \multicolumn{6}{l}{NCPD-}    
 && \multicolumn{4}{l}{NTD-} 
 &&   \multirow{2}{*}{NLRT}   \\
 (dB)&                      && MU       & HALS  & APG   & BCD   & CDTF  & SaCD  && MU   & HALS  & APG   & BCD     &&                           \\\midrule
 \multirow{2}{*}{30} &\multirow{2}{*}{  3.16} &&  2.98&  2.77&  2.74&  2.76&  2.77&  3.08&&  2.48&  2.48&  2.47&  2.48&&\multirow{2}{*}{\bf   2.48}\\
                         &&&  (0.07)&  (0.01)&  (0.00)&  (0.01)&  (0.01)&  (0.17)&&  (0.00)&  (0.01)&  (0.00)&  (0.00)&&                            \\\midrule
 \multirow{2}{*}{40} &\multirow{2}{*}{  1.00} &&  1.11&  0.89&  0.87&  0.90&  0.89&  1.63&&  0.83&  0.81&  0.81&  0.81&&\multirow{2}{*}{\bf   0.80}\\
                         &&&  (0.06)&  (0.02)&  (0.00)&  (0.03)&  (0.02)&  (0.33)&&  (0.07)&  (0.01)&  (0.01)&  (0.01)&&                            \\\midrule
 \multirow{2}{*}{50} &\multirow{2}{*}{  0.32} &&  0.75&  0.38&  0.28&  0.35&  0.39&  1.19&&  0.28&  0.28&  0.27&  0.28&&\multirow{2}{*}{\bf   0.25}\\
                         &&&  (0.09)&  (0.03)&  (0.01)&  (0.05)&  (0.09)&  (0.38)&&  (0.02)&  (0.01)&  (0.01)&  (0.03)&&      \\\bottomrule
\end{tabular}%
\label{tab-test1}
\end{table}
\begin{table}[!t]
\renewcommand\arraystretch{0.8}
\setlength{\tabcolsep}{3pt}
\centering\scriptsize
\caption{  The averaged running time (in seconds) of different methods in Case 1.}\vspace{-1.5mm}
\begin{tabular}{cc ccccccccc ccccc cccc}\toprule
 \multicolumn{7}{c}{Tensor size: $100\times 100\times 100$}&\multicolumn{8}{c}{Multilinear rank: $[5,5,5]$}\\\midrule
 SNR && \multicolumn{6}{l}{NCPD-}    
 && \multicolumn{4}{l}{NTD-} 
 &&   \multirow{2}{*}{NLRT}   \\
 (dB)& & MU       & HALS  & APG   & BCD   & CDTF  & SaCD  && MU   & HALS  & APG   & BCD     &&                           \\\midrule
 30  & &    6.6 &   0.5 &   5.0 &   2.2 & 1.1 &   6.1 &&  12.2 &  12.7 &  16.6 &   5.4 &&    0.5\\
 40  & &    6.4 &   0.5 &   5.0 &  13.0 &14.7 &   6.3 &&  12.1 &  12.8 &  16.7 &   5.4 &&    0.4\\
 50  & &    6.6 &   0.5 &   9.2 &  13.0 &15.3 &   6.8 &&  12.2 &  12.9 &  16.6 &   5.5 &&    0.3\\\toprule
\multicolumn{7}{c}{Tensor size: $50\times 50\times 50\times 50$}&\multicolumn{8}{c}{Multilinear rank: $[3,3,3,3]$}\\\midrule
 SNR && \multicolumn{6}{l}{NCPD-}    
 && \multicolumn{4}{l}{NTD-} 
 &&   \multirow{2}{*}{NLRT}   \\
 (dB)& & MU       & HALS  & APG   & BCD   & CDTF  & SaCD  && MU   & HALS  & APG   & BCD     &&                           \\\midrule
 30  & &   61.5 &  40.3 & 108.2 & 112.0 &139.7 &  36.2 &&  16.2 &  16.0 &  23.1 &  33.7 &&   11.9\\
 40  & &   60.2 &  39.0 & 106.9 & 112.3 &137.6 &  35.8 &&  16.3 &  15.9 &  23.1 &  41.1 &&    8.3\\
 50  & &   60.2 &  47.9 & 106.3 & 103.0 &147.0 &  36.2 &&  16.0 &  16.1 &  22.7 &  40.9 &&    5.8\\\toprule
\multicolumn{7}{c}{Tensor size: $30\times 30\times 30\times 30\times 30$}&\multicolumn{8}{c}{Multilinear rank: $[2,2,2,2,2]$}\\\midrule
 SNR && \multicolumn{6}{l}{NCPD-}    
 && \multicolumn{4}{l}{NTD-} 
 &&   \multirow{2}{*}{NLRT}   \\
 (dB)& & MU       & HALS  & APG   & BCD   & CDTF  & SaCD  && MU    & HALS   & APG   & BCD     &&                           \\\midrule
 30  & &  249.4 & 159.7 & 215.4 & 218.7 &195.3 & 127.2    && 115.0 &  119.6 & 120.4 & 102.6 &&  106.2\\
 40  & &  219.5 & 192.7 & 150.7 & 215.3 &224.4 & 130.4    && 112.6 &  117.3 & 118.9 & 123.3 &&   78.6\\
 50  & &  233.4 & 184.1 & 127.9 & 243.4 &324.2 & 129.3    && 114.9 &  119.5 & 121.1 & 131.0 &&   56.1\\\bottomrule
\end{tabular}
\label{tab-test1-time}
\end{table}

\begin{figure}[!t]
\scriptsize\setlength{\tabcolsep}{0.5pt}
\renewcommand\arraystretch{0.8}
\centering
\begin{tabular}{c}
\includegraphics[width=0.70\linewidth]{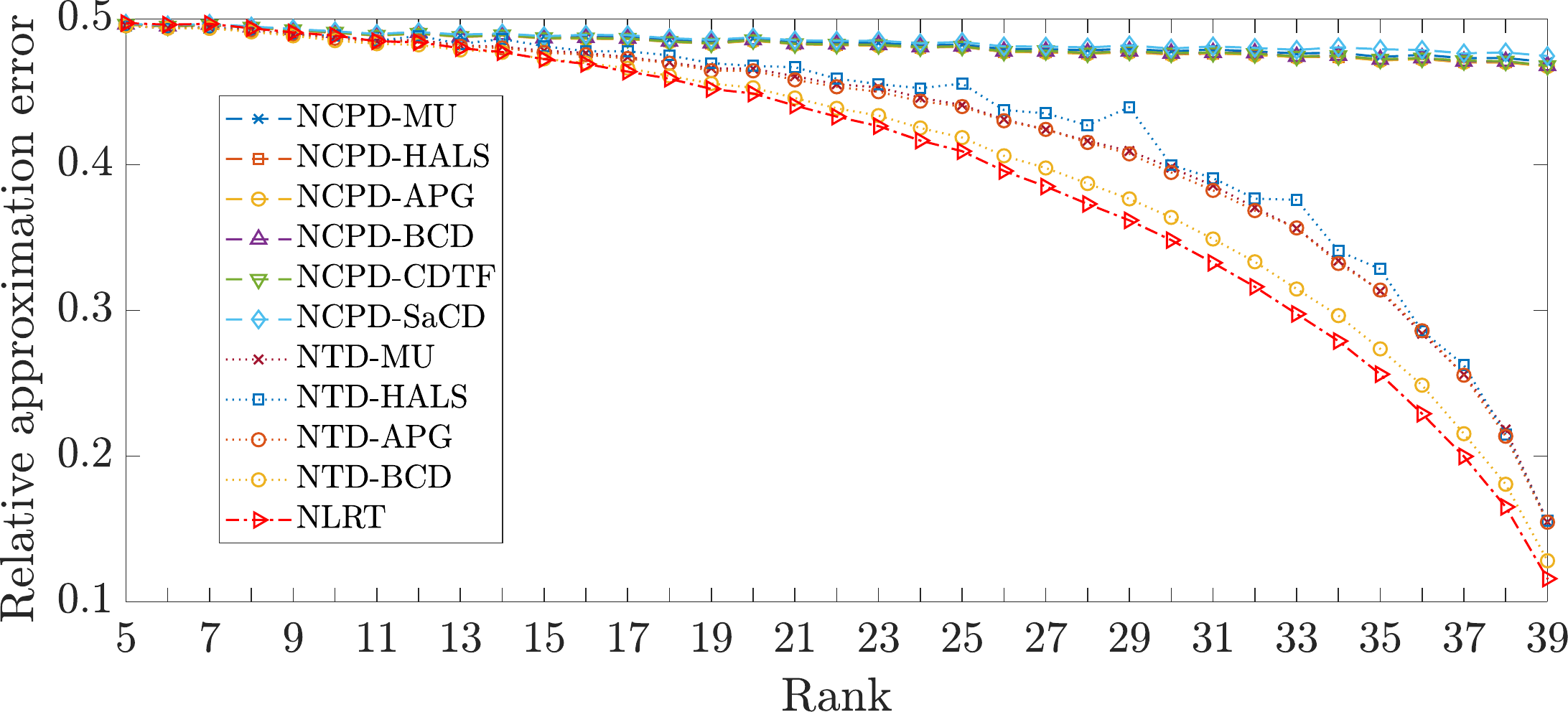}\\
(a) Tensor size: $40\times 40\times 40$\\\\
\includegraphics[width=0.70\linewidth]{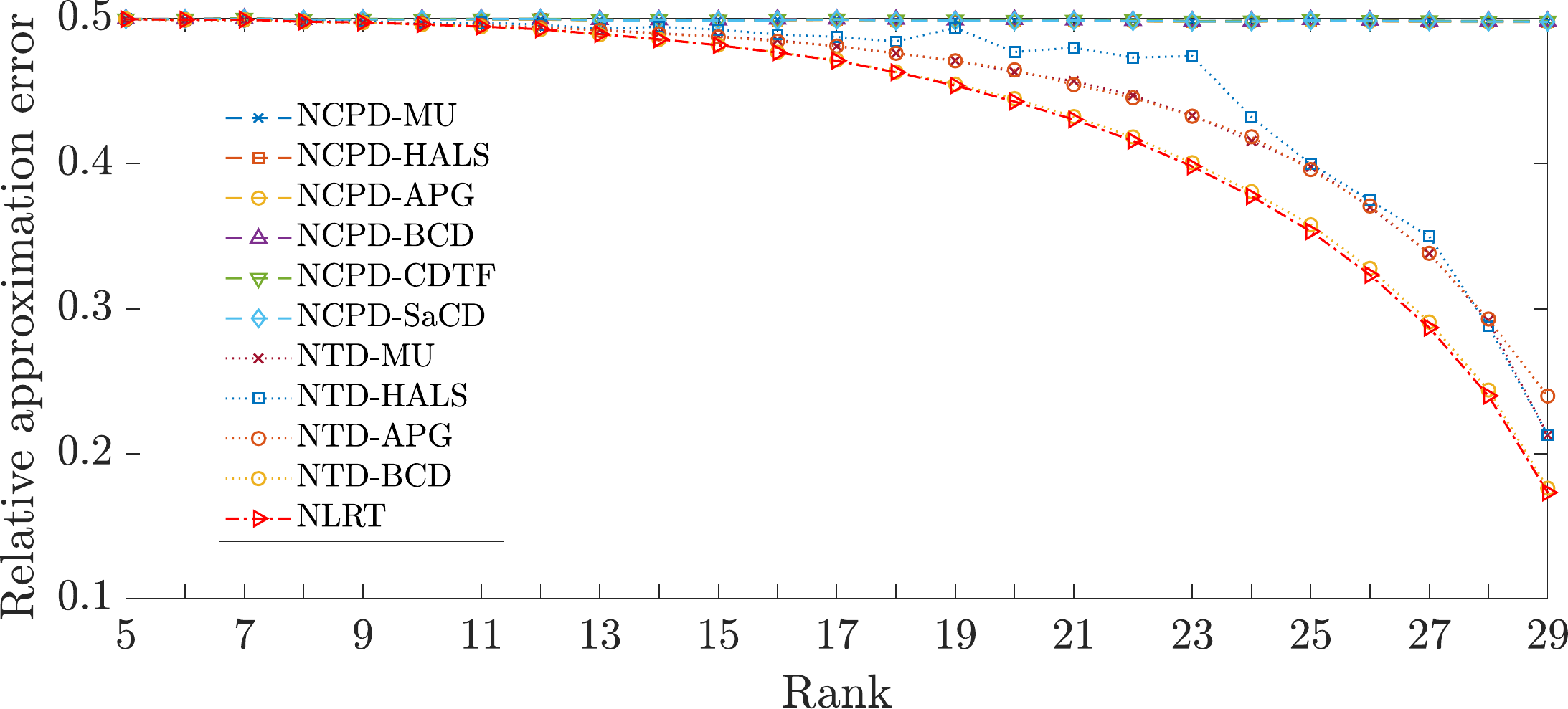}\\
(b)  Tensor size: $30\times 30 \times 30\times30$
\end{tabular}\vspace{-3mm}
\caption{  Relative approximation errors on the randomly generated tensors in Case 2 with respect to the different rank settings.}
  \label{fig-syn-rand2}\vspace{-3mm}
\end{figure}

The relative approximation errors in Case 2 with respect to different values of $r$ are plotted in Fig. \ref{fig-syn-rand2}. As we stated, the tensor of a given size will be fixed once generated. Then, for different values of $r$, we run each algorithm 10 times and the averaged values are plotted.
From Fig. \ref{fig-syn-rand2}, we can see that the proposed NLRT method and NTD-BCD perform
better than the other methods. For the tensors of the size $40\times40\times40$, the superior of our method over NTD-BCD is obvious when the rank is in between 27 and 39.

\begin{figure}[!t]
\setlength{\tabcolsep}{3pt}
\renewcommand\arraystretch{0.5}
\centering
\begin{tabular}{ccc}

\includegraphics[width=0.48\linewidth]{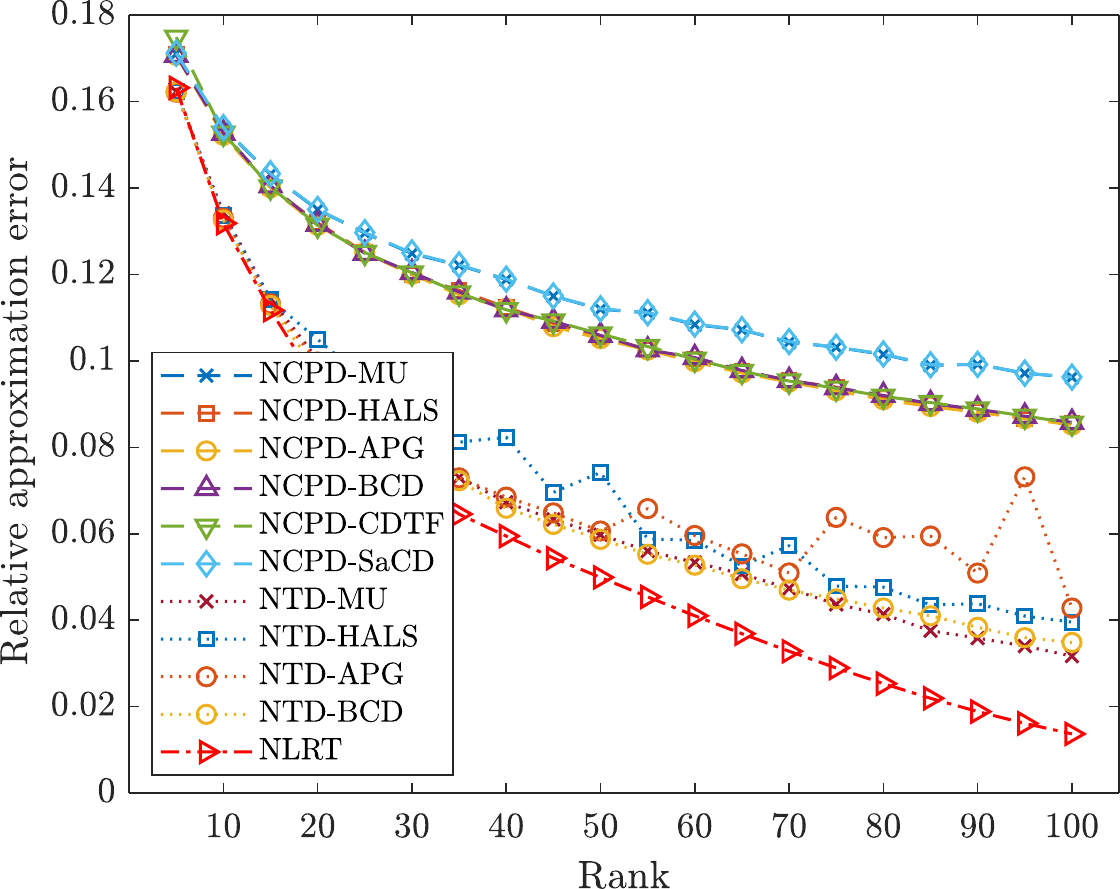}&&
\includegraphics[width=0.48\linewidth]{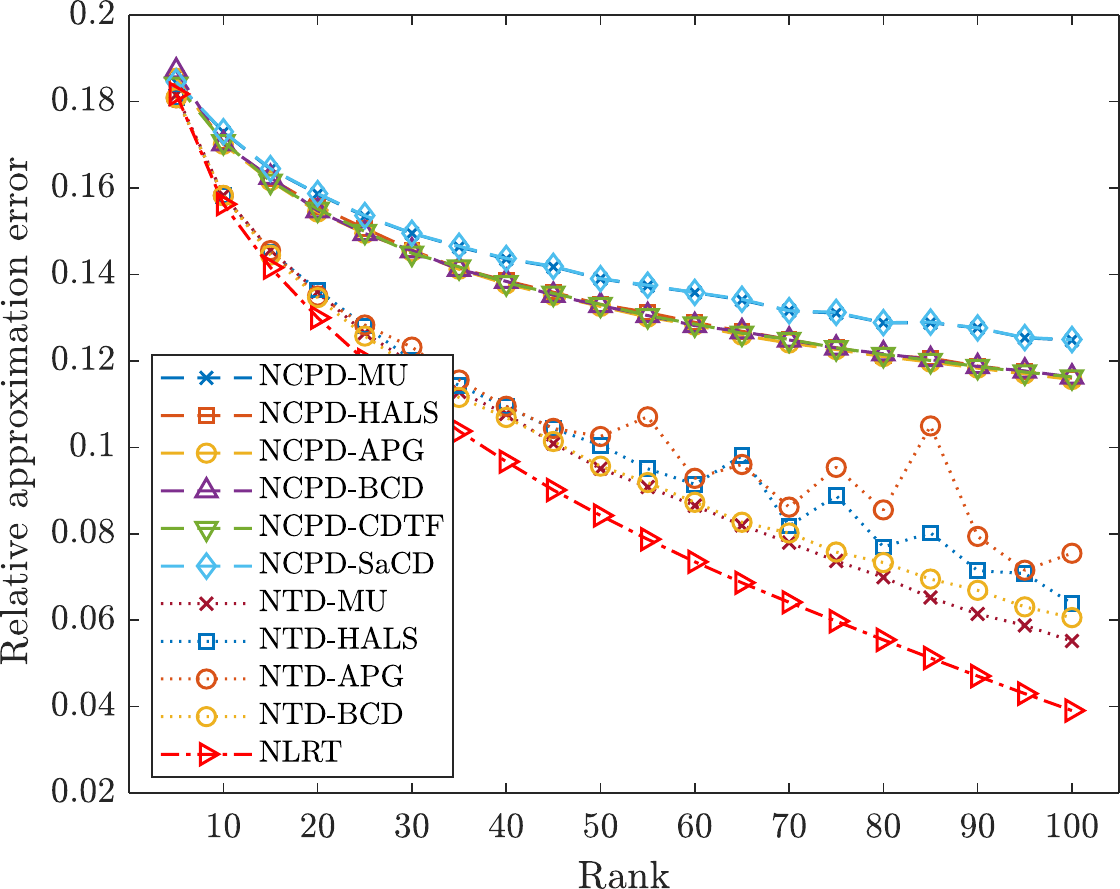}\\
(a) ``foreman''&& (b) ``coastguard''\\\\\\
\includegraphics[width=0.48\linewidth]{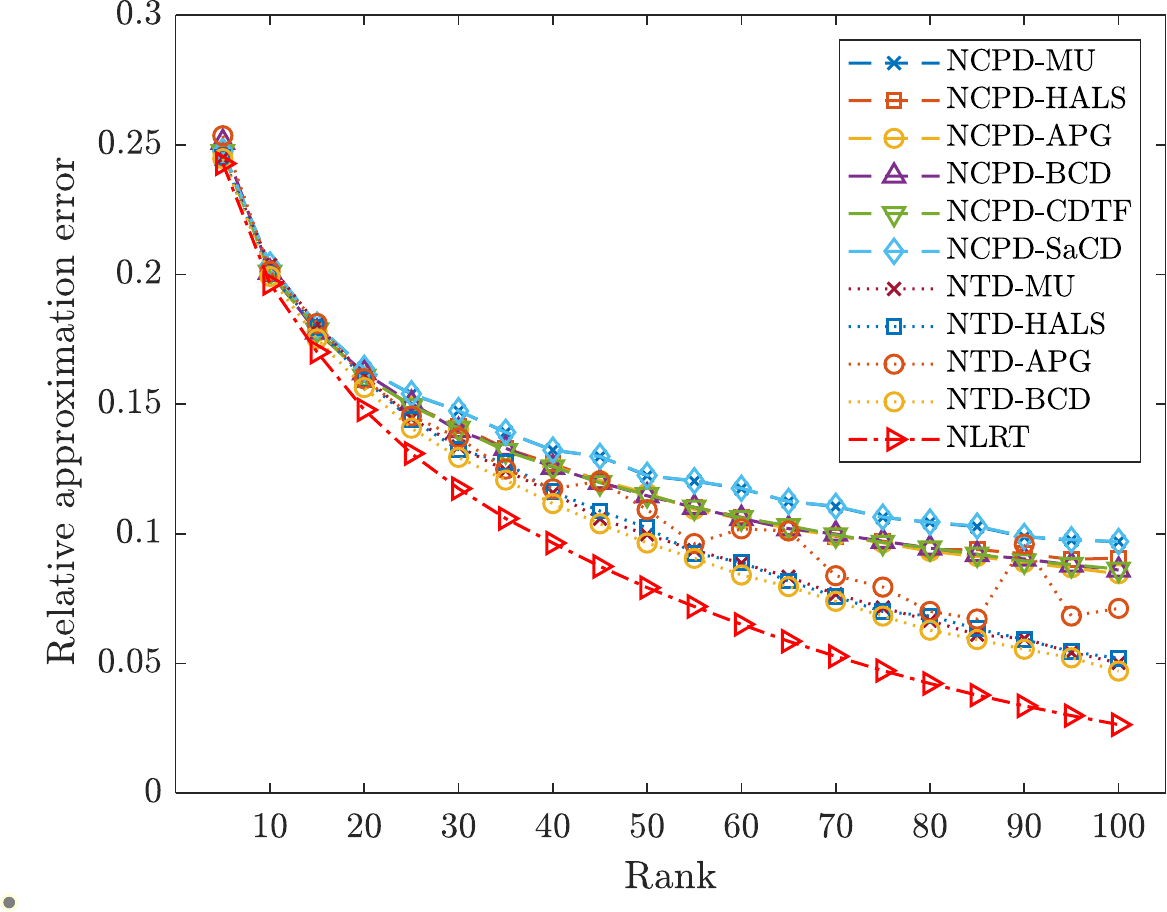}&&
\includegraphics[width=0.48\linewidth]{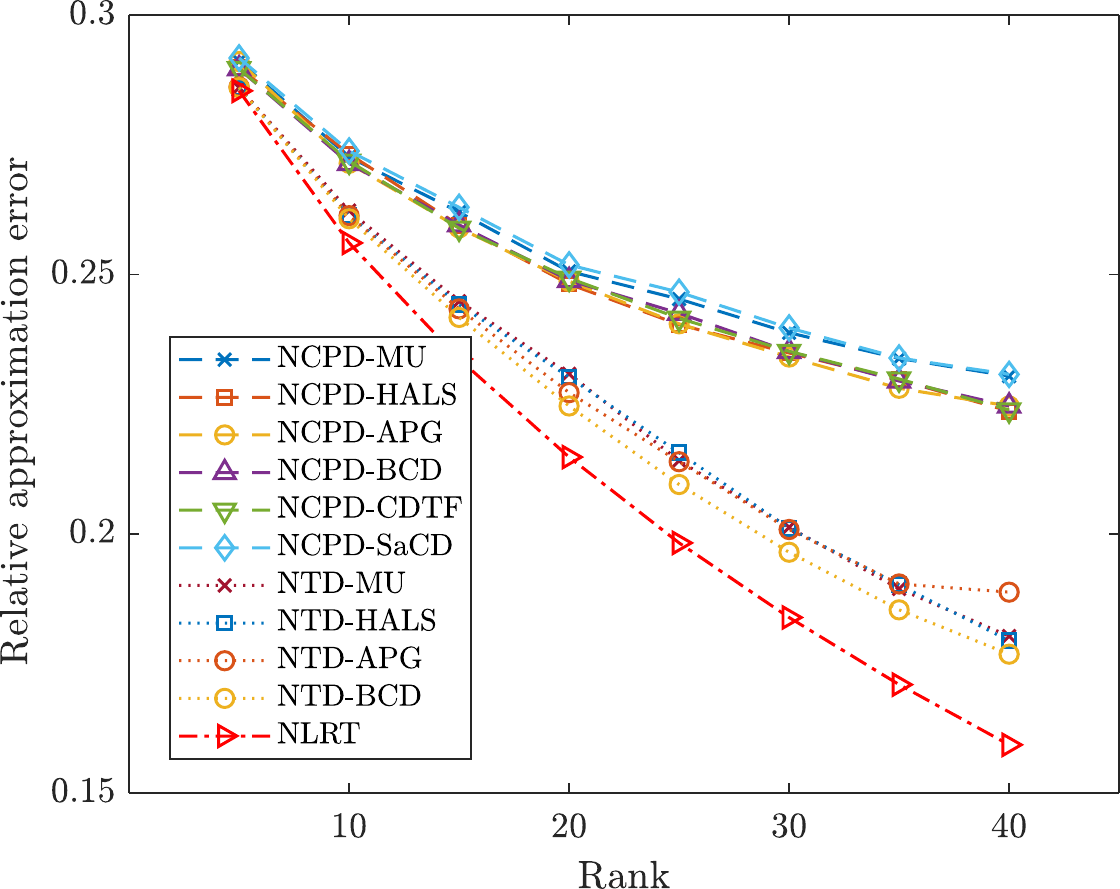}\\
(c) ``news''&& (d) ``basketball''
\end{tabular}\vspace{-3mm}
\caption{  Relative approximation errors on 4 videos (100 frames) with respect to the different rank settings.}
  \label{fig-video-approxi}
\end{figure}

\subsection{Video Data}

In this subsection, we select 5 videos\footnote{Videos are available at \url{http://trace.eas.asu.edu/yuv/} and \url{https://sites.google.com/site/jamiezeminzhang/publications}.} to test our method on the task of approximation.
Three videos (respectively named ``foreman'', ``coastguard'', and ``news'') are of the size $144\times176\times100$ (height$\times$width$\times$frame) and one (named ``basketball'') is of the size $44\times256\times40$.   One long video (named ``bridge-far'') of the size $144\times176\times2000$ is also selected to test the approximation ability for large scale data.
Firstly, we set the multilinear rank to be $(r,r,\cdots,r)$ and the CP rank to be $r$.
We test our method to approximate these five videos with varying $r$ from 5 to 100.
Moeover, we add the Gaussian noise to the video ``coastguard'' with different noise levels
($\text{SNR}_\text{dB}$ = 20, 30 ,40, 50), and test the approximation ability of differen methods for the noisy video data.

\begin{figure}[!t]
\setlength{\tabcolsep}{3pt}
\renewcommand\arraystretch{0.5}
\centering
\begin{tabular}{ccc}
\multicolumn{3}{c}{\includegraphics[width=0.69\linewidth]{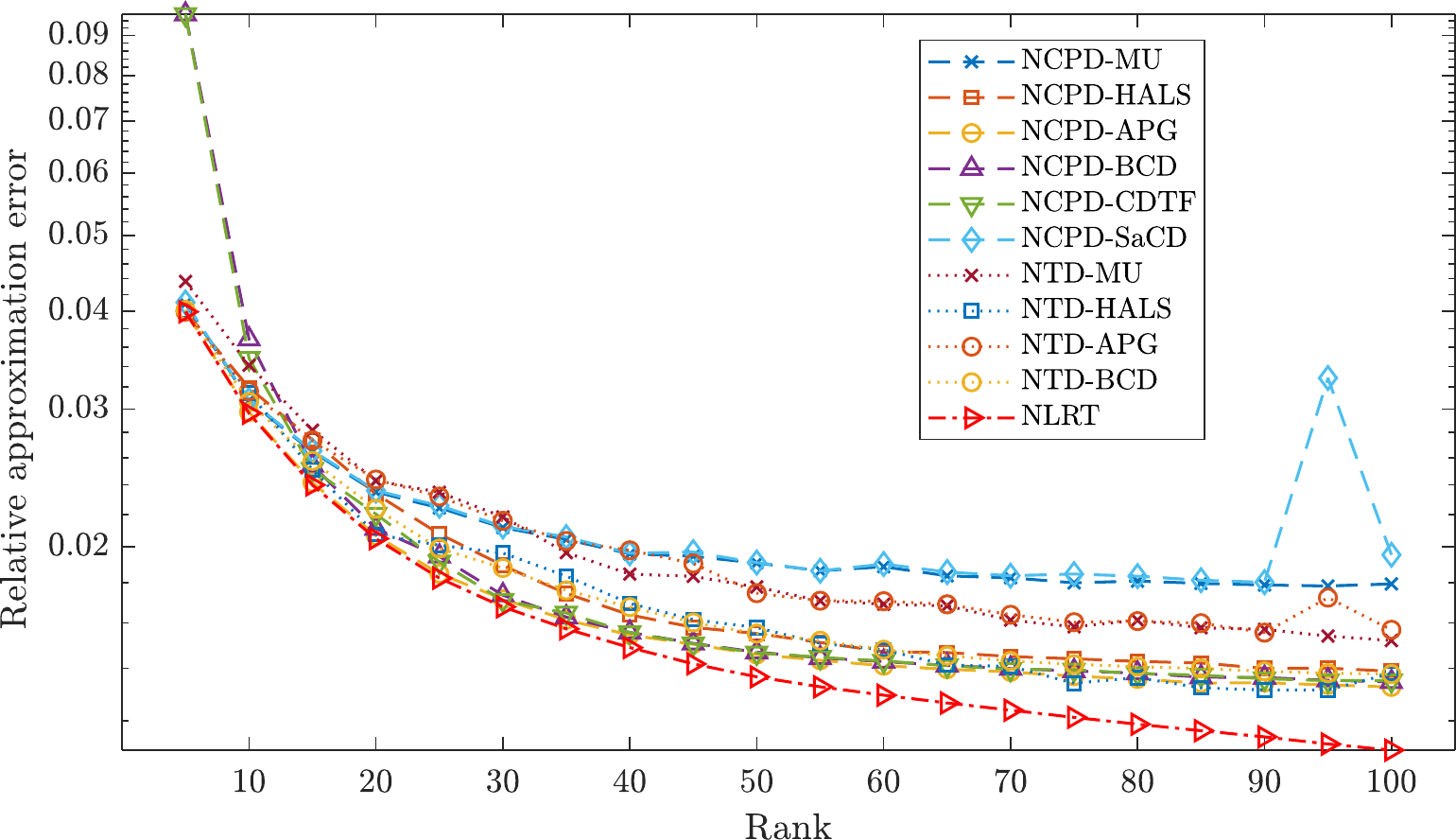}}\\
\end{tabular}\vspace{-3mm}
\caption{  Relative approximation errors on the video ``bridge-far'' (2000 frames) with respect to the different rank settings.}
  \label{fig-video-approxi2}
\end{figure}

\begin{figure}[!t]
\setlength{\tabcolsep}{5pt}
\renewcommand\arraystretch{0.5}
\centering
\begin{tabular}{cc}
\includegraphics[width=0.48\linewidth]{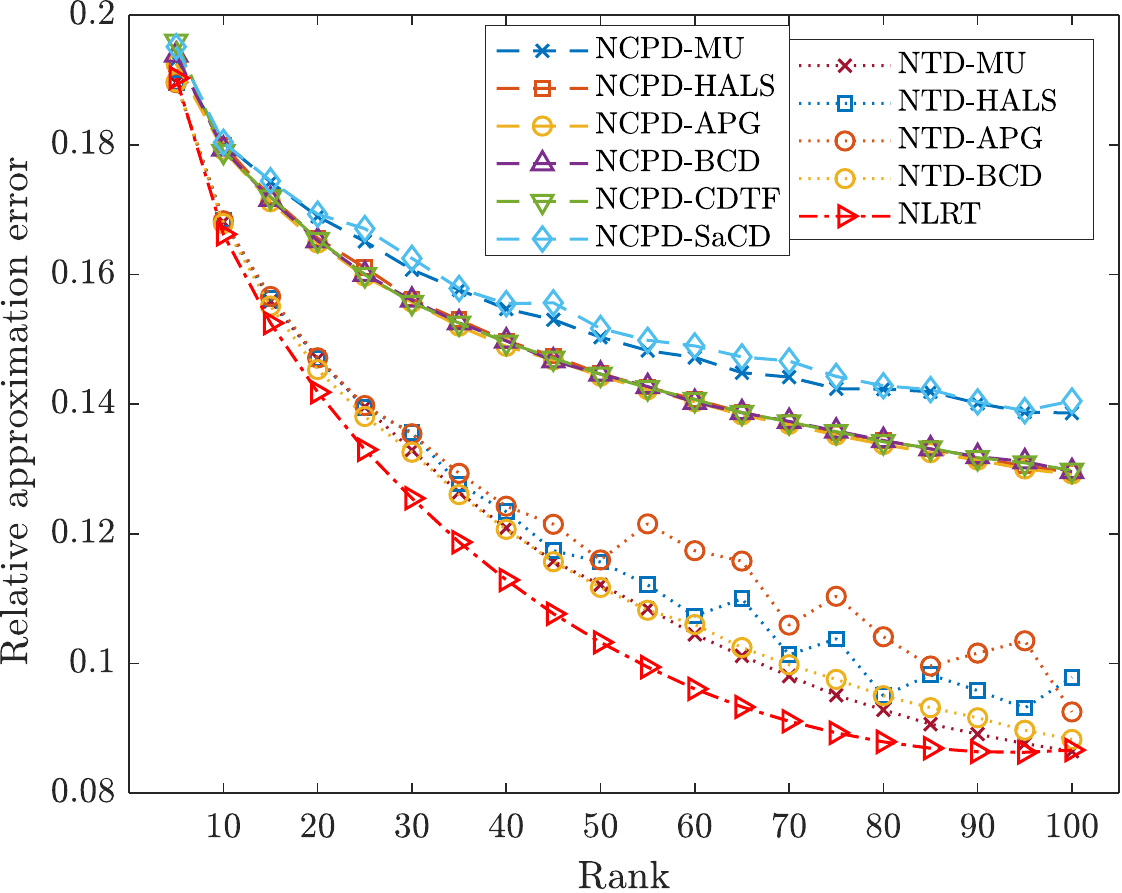}&
\includegraphics[width=0.48\linewidth]{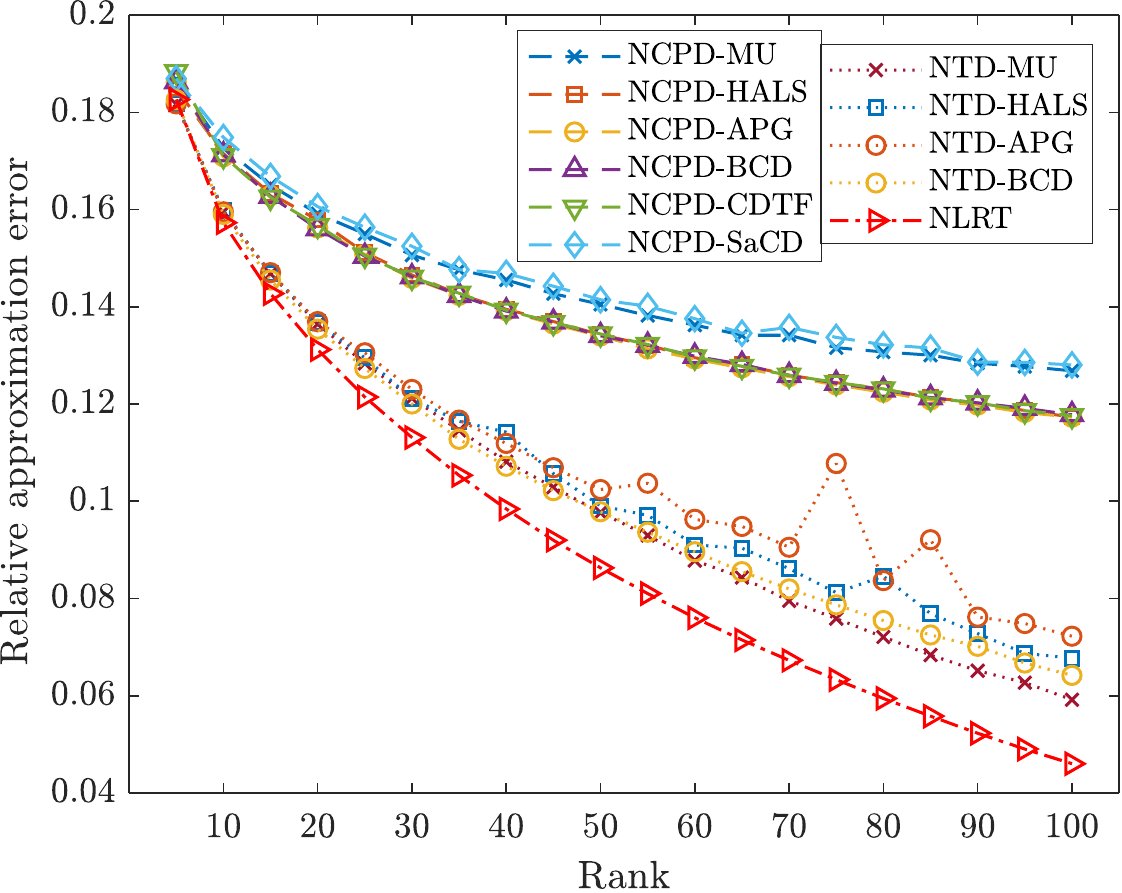}\\
(a) SNR = 20 dB & (b) SNR = 30 dB\\\\\\
\includegraphics[width=0.48\linewidth]{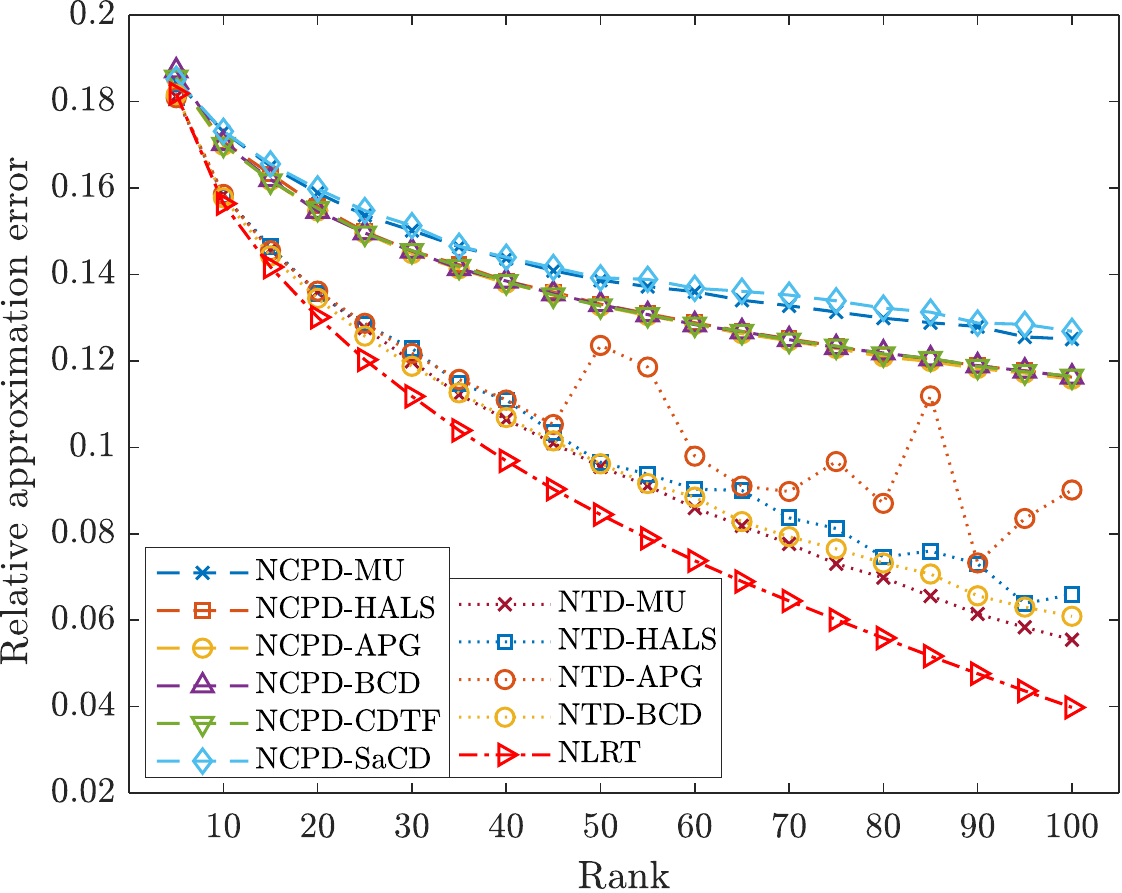}&
\includegraphics[width=0.48\linewidth]{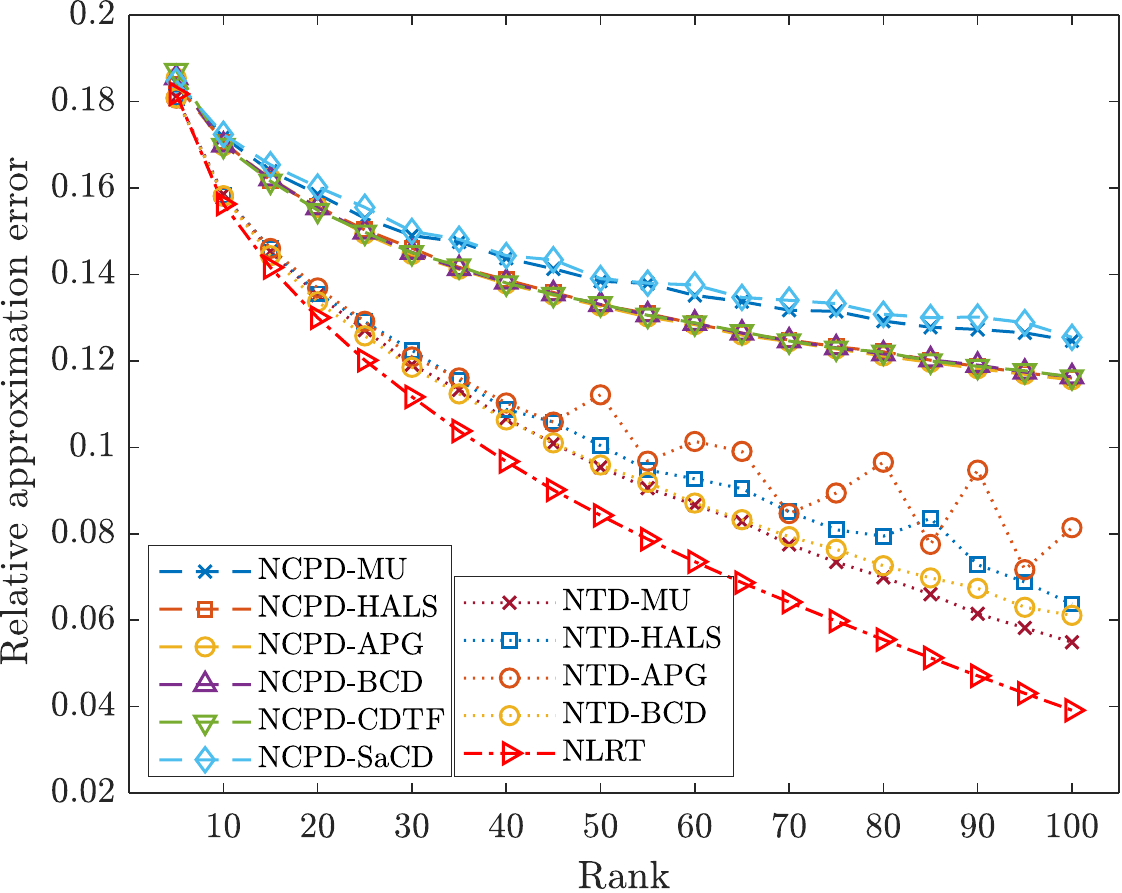}\\
(c) SNR = 40 dB & (d) SNR = 50 dB\\
\end{tabular}\vspace{-1.5mm}\vspace{-1.5mm}
\caption{ Relative approximation errors on the noisy video ``coastguard'' with respect to different rank settings and different noise levels.}
  \label{fig-noisy-video-approxi}
\end{figure}

We plot the relative approximation errors with respect to $r$ on 5 videos in Figs. \ref{fig-video-approxi} and \ref{fig-video-approxi2}. Although, for some videos the approximation errors of the results by NCPD methods are much higher than others, owing to that setting CP rank as $r$ largely constrained the model representation ability, we can still see that the potential of NCPD methods are promising. For example, for the videos ``news'' and ``bridge-far'', NCPD methods are even occasionally superior to NTD methods. Thus, the comparison with NCPD methods provides some insights. From Figs. \ref{fig-video-approxi} and \ref{fig-video-approxi2}, it can be seen that the approximation errors of the results by our method are the lowest. Fig. \ref{fig-noisy-video-approxi} shows the relative approximation errors on the noisy video ``coastguard'' with respect to  $r$. Similarly, our method achieves the lowest approximation errors on the video ``coastguard'' with respect to different rank settings and different noise levels. In Table \ref{tab-video-time}, we list the average running time of each method.

\begin{table}[!t]
\renewcommand\arraystretch{1.0}
\setlength{\tabcolsep}{1.5pt}
\centering\scriptsize
\caption{  The average running time (in seconds) of different methods on video data.}\vspace{-1.5mm}
\begin{tabular}{cc ccccccccc ccccc cccc}\toprule
\multirow{2}{*}{Video} &$\#$& \multicolumn{6}{l}{NCPD-}    
 && \multicolumn{4}{l}{NTD-} 
 &&   \multirow{2}{*}{NLRT}   \\
&frames                 & MU     & HALS  & APG   & BCD   & CDTF  & SaCD  && MU    & HALS  & APG   & BCD   &&       \\\midrule
``foreman''      &100   &     60 &    66 &    55 &    24 &    23 &    45 &&   202 &   188 &   354 &    74 &&     20\\
``news''         &100   &     46 &    46 &    37 &    17 &    16 &    33 &&   176 &   197 &   313 &    59 &&     25\\
``coastguard''   &100   &     35 &    36 &    30 &    13 &    12 &    27 &&   129 &   165 &   228 &    46 &&     15\\
``basketball''   &40    &     16 &    17 &    11 &     3 &     3 &    12 &&    25 &    20 &    34 &    14 &&     15\\
 ``bridge-far''  &2000  &    386 &   173 &   265 &   186 &   188 &   209 &&   183 &   211 &   299 &   511 &&    296\\
\toprule
\multirow{2}{*}{Video} &SNR& \multicolumn{6}{l}{NCPD-}    
 && \multicolumn{4}{l}{NTD-} 
 &&   \multirow{2}{*}{NLRT}   \\
&(dB)                   & MU     & HALS  & APG   & BCD   & CDTF  & SaCD  && MU    & HALS  & APG   & BCD   &&       \\\midrule
\multirow{4}{*}{``coastguard''}
                &   20  &     28 &    29 &    28 &     9 &     9 &    20 &&   135 &   146 &   258 &    46 &&     18\\
                &   30  &     28 &    29 &    29 &    10 &    10 &    20 &&   133 &   145 &   254 &    46 &&     17\\
                &   40  &     28 &    29 &    29 &    10 &    10 &    20 &&   134 &   147 &   255 &    47 &&     16\\
                &   50  &     28 &    29 &    29 &    10 &     9 &    20 &&   136 &   143 &   259 &    46 &&     15\\\bottomrule
\end{tabular}
\label{tab-video-time}
\end{table}

\subsection{Hyperspectral Data}\label{sec:hyper}

In this subsection, we test different methods on the hyperspectral data.
We consider four hyperspectral images (HSIs):
a subimage of Pavia City Center dataset\footnote{\scriptsize Data available at \url{http://www.ehu.eus/ccwintco/index.php?title=Hyperspectral\_Remote\_Sensing\_Scenes}.} of the size $200\times200\times80$ (height$\times$width$\times$spectrum),
a subimage of Washington DC Mall dataset\footnote{\scriptsize Data available at \url{https://engineering.purdue.edu/\~biehl/MultiSpec/hyperspectral.html}.} of the size $200\times200\times160$,
the RemoteImage\footnote{\scriptsize Data available at \url{https://www.cs.rochester.edu/~jliu/code/TensorCompletion.zip}.} of the size $200\times200\times89$,
and a subimage of Curprite dataset\footnote{\scriptsize Data available at \scriptsize \url{https://aviris.jpl.nasa.gov/data/free\_data.html}.} of the size $150\times150\times150$.
 Meanwhile, a hyperspectral video (HSV)\footnote{\scriptsize Data available at \url{http://openremotesensing.net/knowledgebase/hyperspectral-video/}.}  of the size $120\times188\times33\times31$ (height$\times$width$\times$spectrum$\times$time) is also selected to test the effectiveness of different methods on the fourth order tensor.

\begin{figure}[!t]
\scriptsize\setlength{\tabcolsep}{5pt}
\centering\renewcommand\arraystretch{0.8}
\begin{tabular}{cc}
\includegraphics[width=0.48\linewidth]{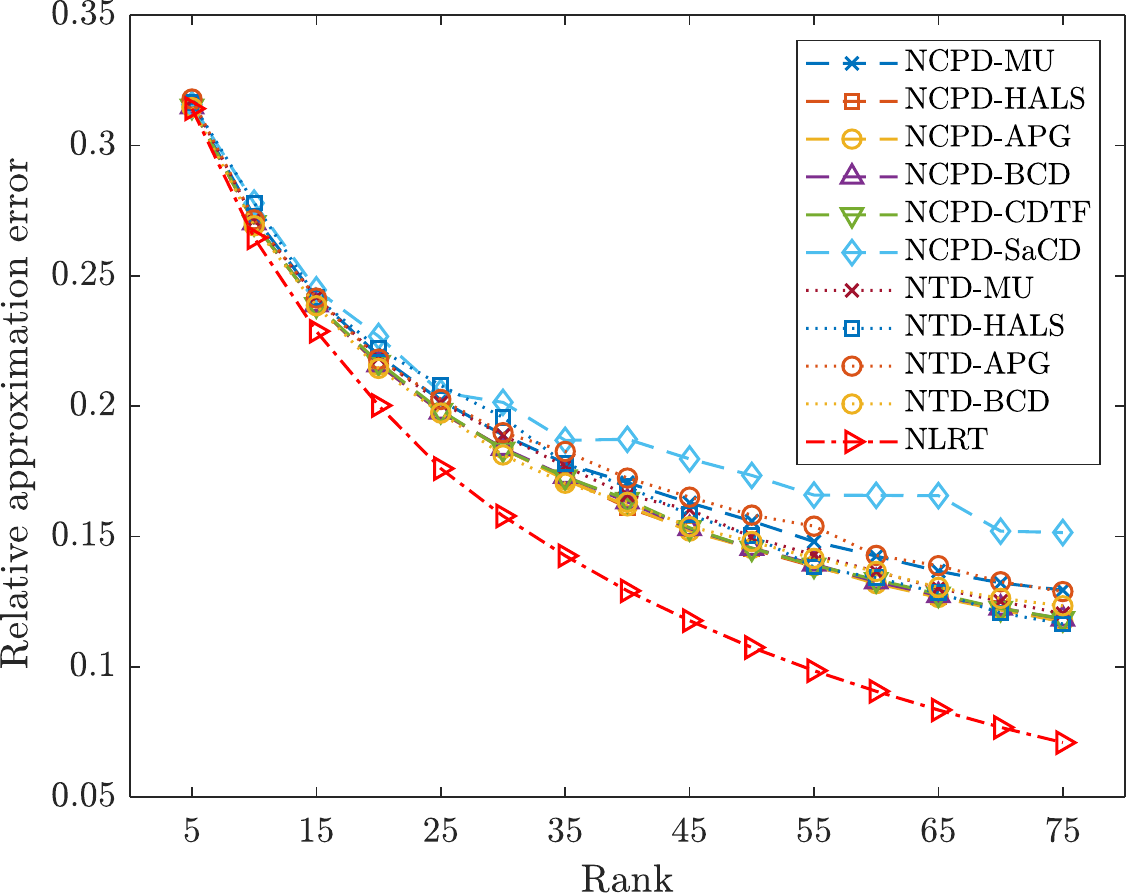}&
\includegraphics[width=0.48\linewidth]{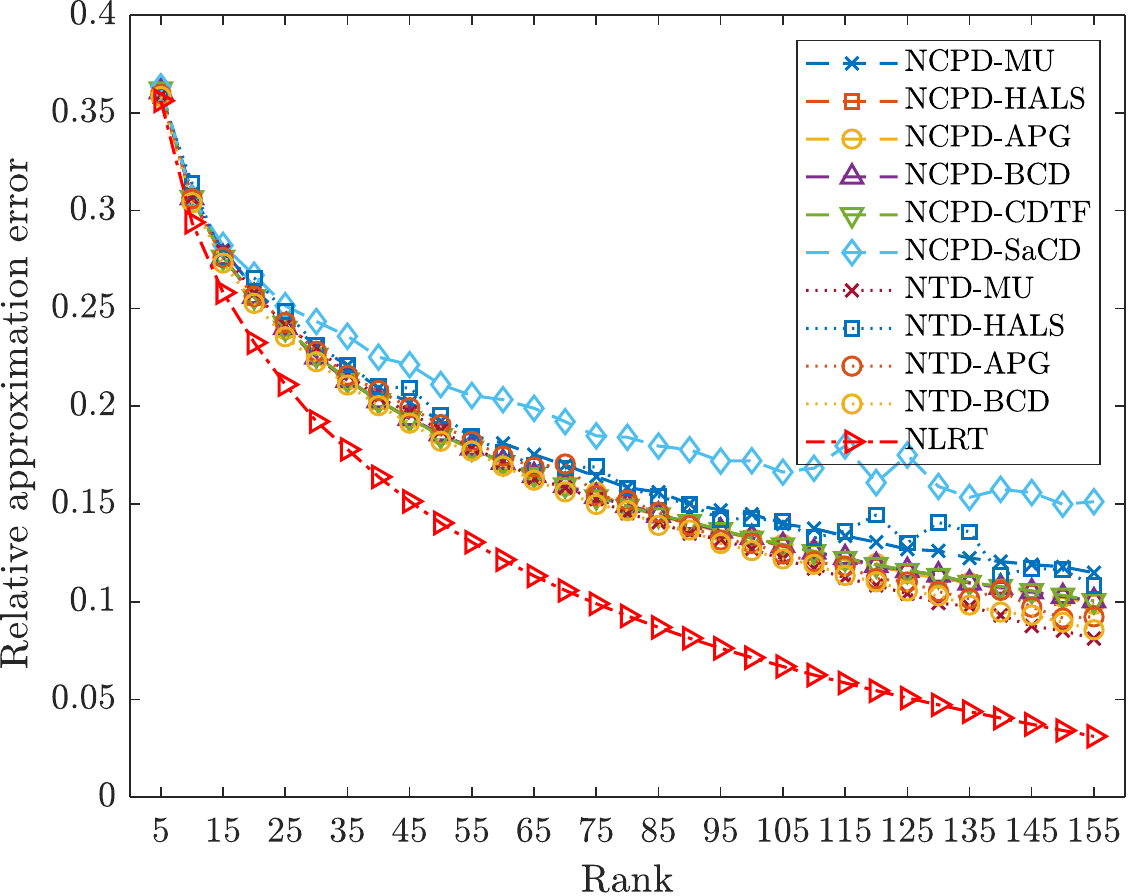}\\
(a) Pavia City Center&(b) Washington DC Mall\\\\\\
\includegraphics[width=0.48\linewidth]{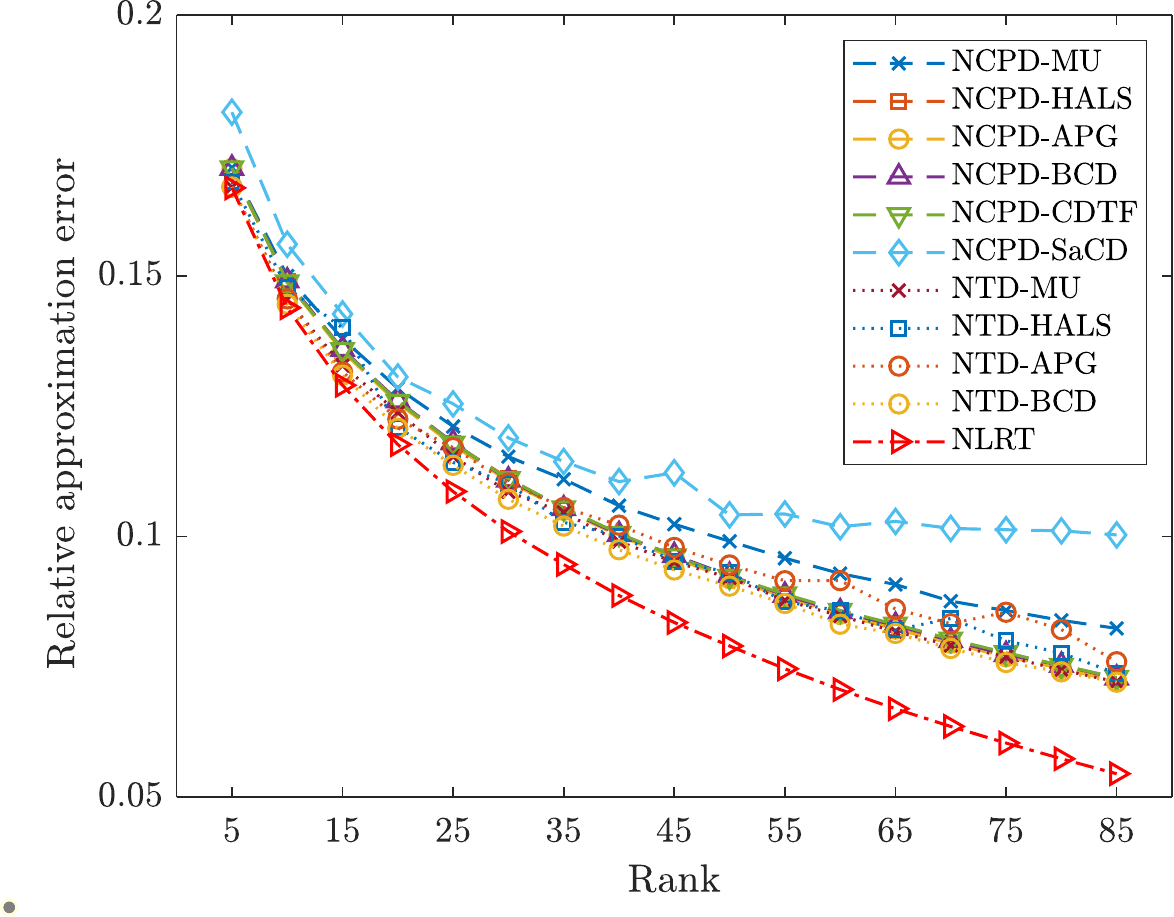}&
\includegraphics[width=0.48\linewidth]{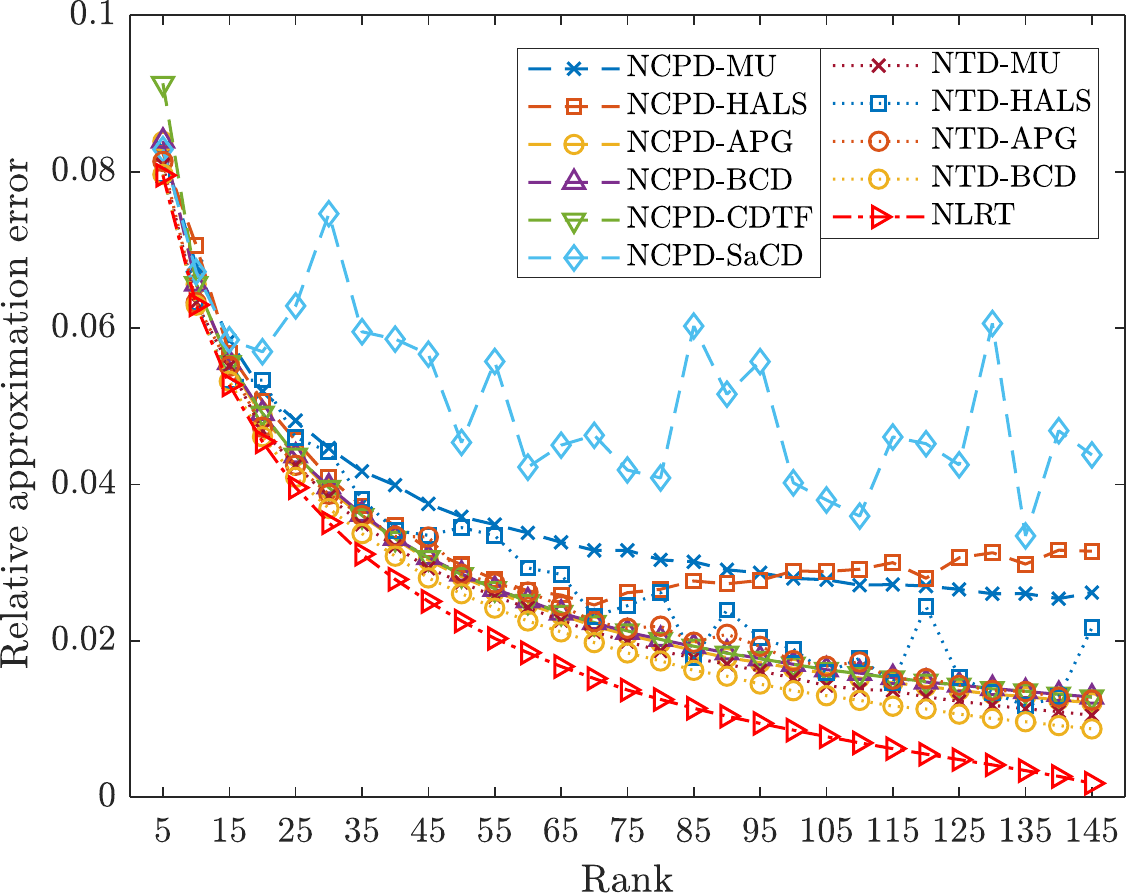}\\
(c) RemoteImage&(d) Curprite\\
\end{tabular}
\caption{  Relative approximation errors on 4 HSIs with respect to the different rank settings.}
  \label{fig-hsi-approxi}
\end{figure}

Figs. \ref{fig-hsi-approxi} and \ref{fig-hsv-approxi} report the relative approximation errors with respect to different values of rank $r$, i.e., multilinear rank = ($r,r,r$) or ($r,r,r,r$) and CP rank = $r$.
It is evidently that the relative approximation errors by our NLRT are the lowest among all the methods.
It is interesting to note that the difference between our method and NTD-BCD (the second best comparison
method) is more significant than that on the synthetic fourth order tensor data.

\begin{figure}[!t]
\small\setlength{\tabcolsep}{0.5pt}
\centering
\begin{tabular}{c}
\includegraphics[width=0.75\linewidth]{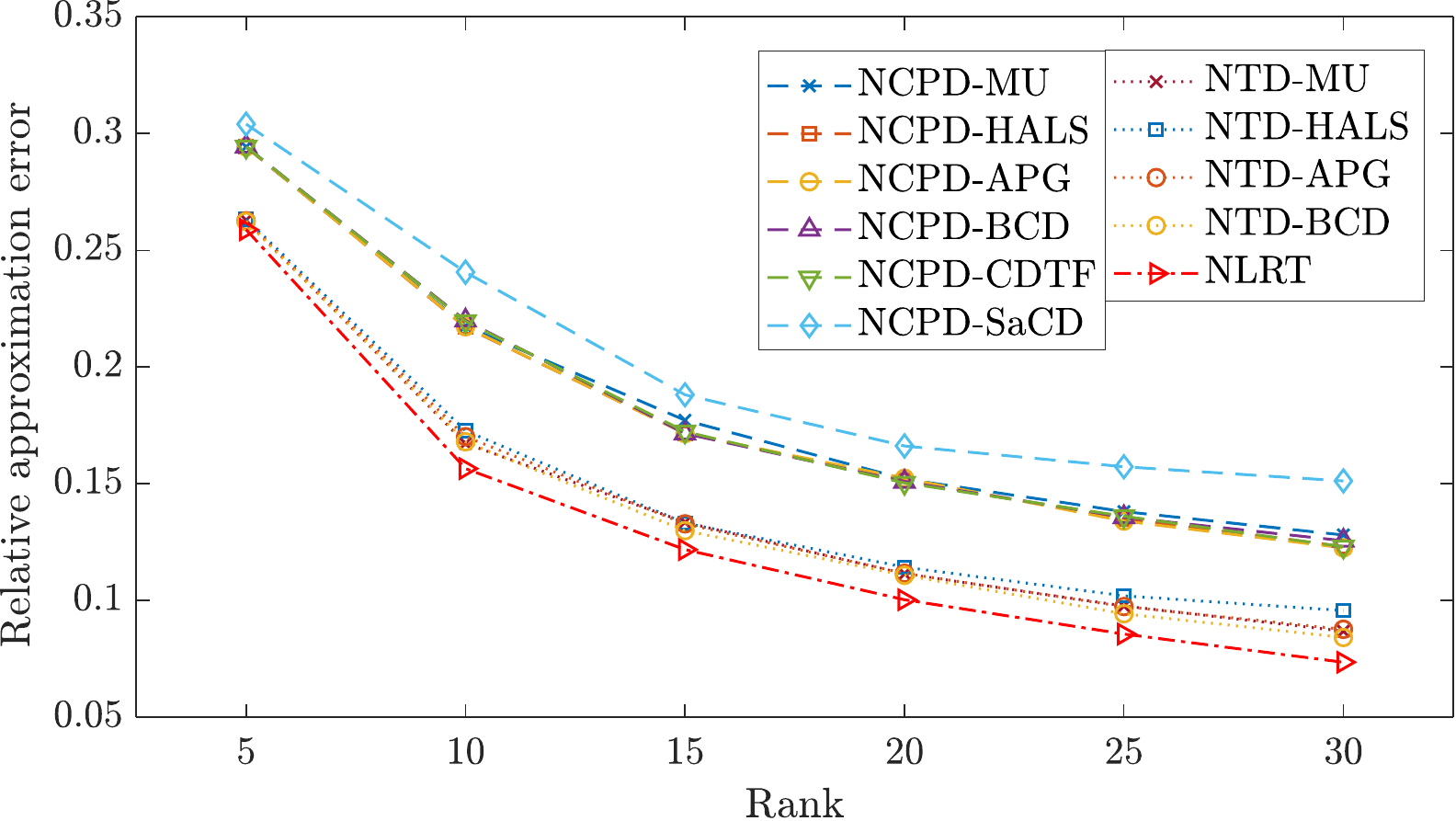}\\
\end{tabular}\vspace{-3mm}
\caption{  Relative approximation errors on the HSV  with respect to the different rank settings.}
  \label{fig-hsv-approxi}
\end{figure}

\begin{figure}[!t]
\setlength{\tabcolsep}{3pt}
\renewcommand\arraystretch{0.8}
\centering%
\includegraphics[width=0.98\textwidth]{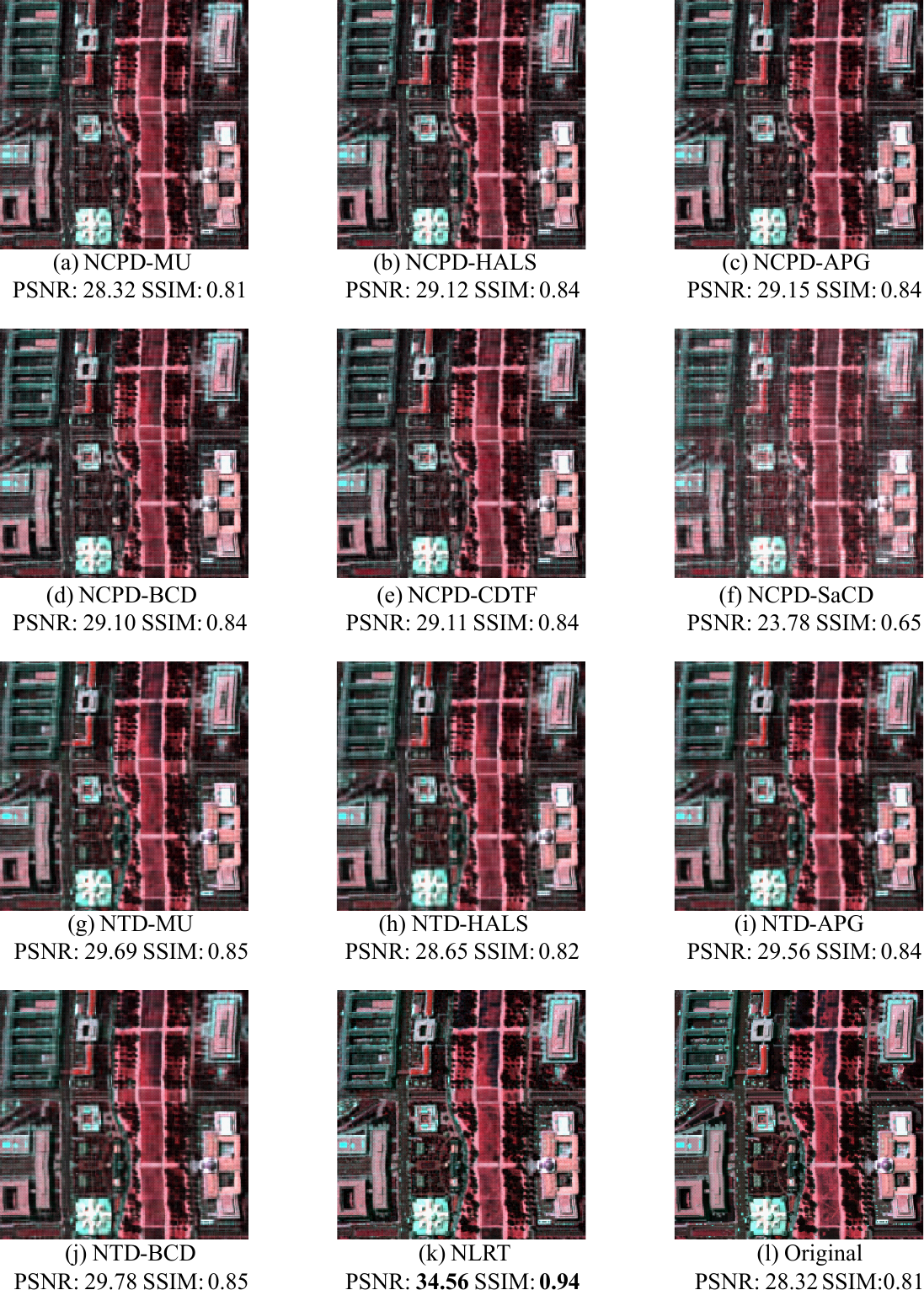}%
%
%
%
%
\caption{The pseudo-color images composed of the 113-th, 2-nd, and 16-th bands of the non-negative low-rank approximations by different methods when setting the rank $100$ on the Washington DC Mall.}
  \label{fig-hsi-DC-band}
\end{figure}

In Fig. \ref{fig-hsi-DC-band}, we display the pseudo-color images of the results on the Washington DC Mall dataset with the multilinear rank (100,100,100) and CP rank = 100.
The pseudo-color image is composed of the 113-th, 2-nd, and 16-th bands as the red, green, and blue channels, respectively. We also
compute two image quality assessments (IQAs):
the peak signal to noise
ratio (PSNR)\footnote{\scriptsize \url{https://en.wikipedia.org/wiki/Peak\_signal-to-noise\_ratio}}
and the structural similarity index  (SSIM) \cite{wang2004image} of all the spectral bands for each band. Higher values of these two indexes indicate better reconstruction quality. In Fig. \ref{fig-hsi-DC-band}, we report the mean values across spectral bands of these two IQAs.
 It can be found in Fig. \ref{fig-hsi-DC-band} that
both visual and quality assessments of the NCPD methods are comparable to NTD methods. The proposed NLRT method largely outperforms other methods in terms of two IQAs, achieving the first place.

\subsection{Selection of Features}\label{Exp:features}

\begin{figure}[!t]
\scriptsize\setlength{\tabcolsep}{0.5pt}
\centering
\includegraphics[width=0.98\linewidth]{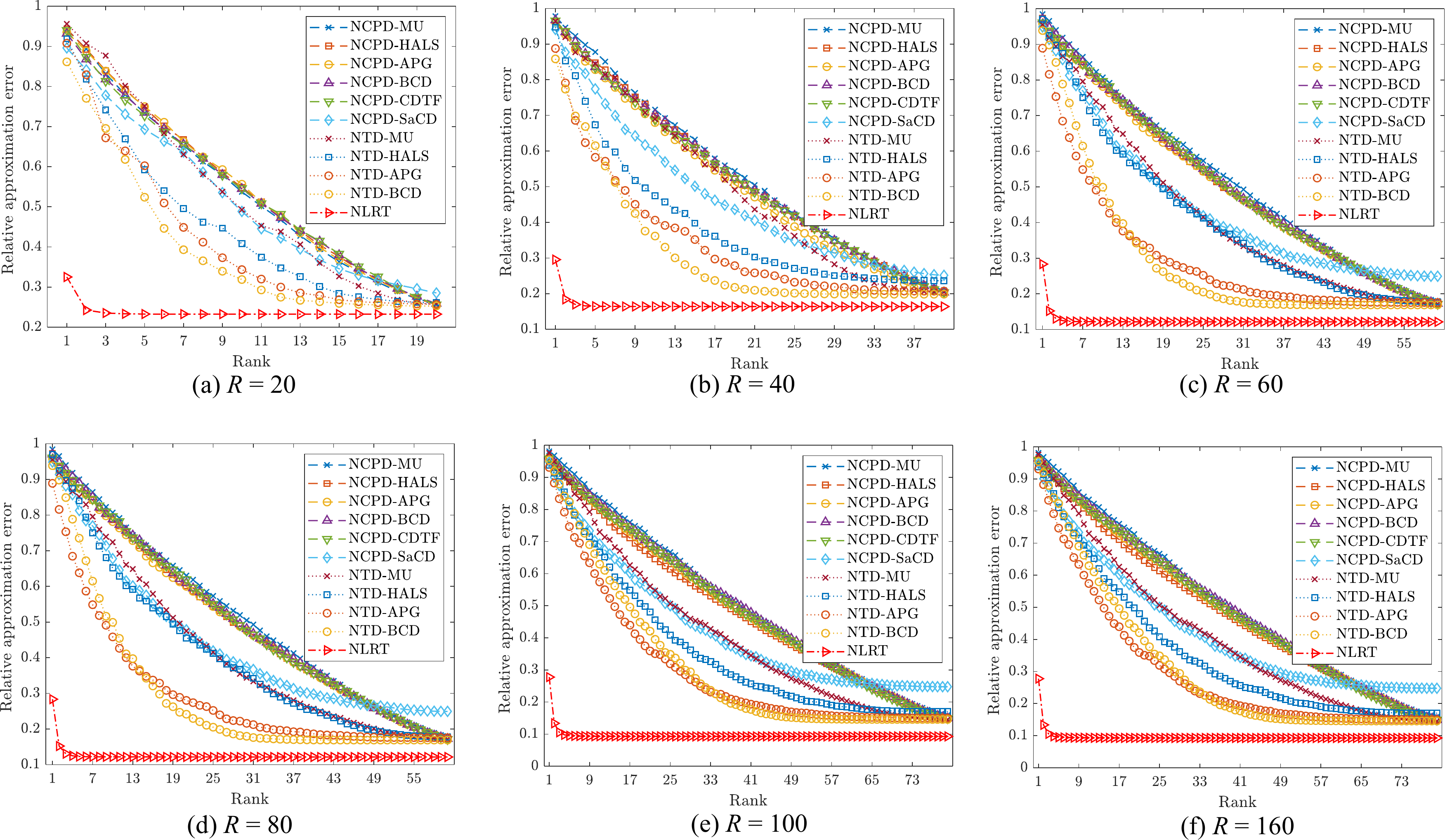}
%
\caption{The comparison of relative residuals with respect to the number of mode-3 components to be used in the tensor approximation with $R=20,40,60,80,160$ for the hyperspectral image Washington DC Mall.}
  \label{fig-hsi-sinval}
\end{figure}

One advantage of the proposed NLRT method is that it can
provide a significant index based on singular values of unfolding matrices
\cite{song2020nonnegative}
that can be used to identify important singular basis vectors in the approximation.
Those singular values and singular vectors are natural concomitants brought out by our algorithm without additional computations of SVD.

Here we take the HSI Washington DC Mall as an example.
We compute the low-rank approximations of the proposed NLRT method and the other comparison
methods with multlinear rank $(r,r,r)$ and CP rank $r$ for $r=20,40,60,80,160$.
For the approximation results by NCPD methods,
\begin{figure}[!t]
\centering
\includegraphics[width=0.98\linewidth]{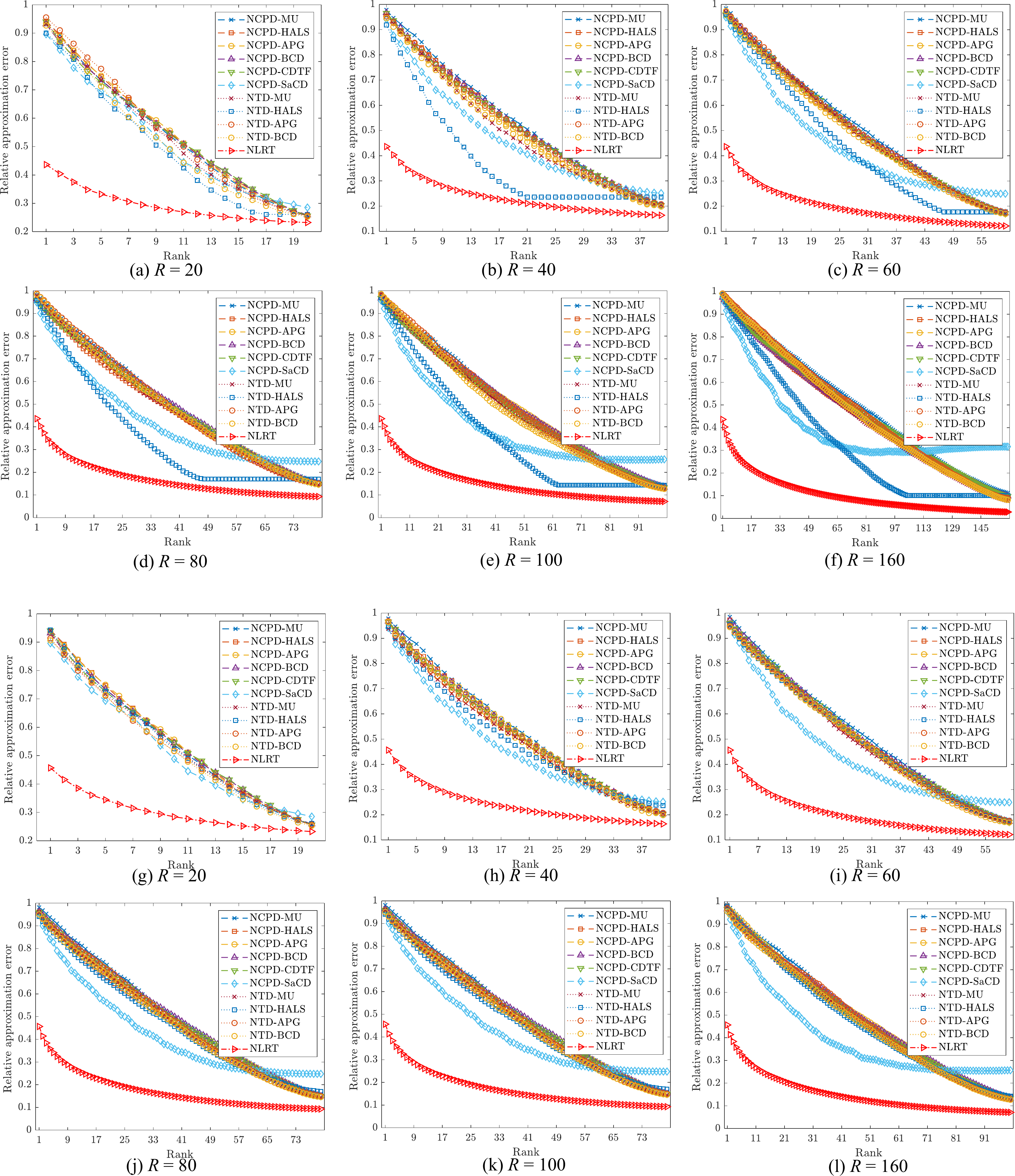}
%
\caption{The comparison of relative residuals with respect to the number of the first mode
(upper two rows from (a) to (f)) and the second mode (bottom two rows from (g) to (l)) components to be used in the tensor approximation with $R=20,40,60,80,160$ for the hyperspectral image Washington DC Mall.}
  \label{fig-hsi-sinval2}
\end{figure}we normalize the base vectors in \ref{CP} such that the $\ell_2$ norms
of $\mathbf{a}^{k,1}$, $\mathbf{a}^{k,2}$ and $\mathbf{a}^{k,3}$ are equal to 1,
and rearrange the resulting values $\lambda'_z$ in the descending order in the CP decomposition.
In Fig. \ref{fig-hsi-sinval}, we plot
$$\|\mathcal{A}-\mathcal{X}_\text{NCPD}(j)\|_F/\|\mathcal{A}\|_F$$ with respect to $j$,
where $\mathcal{X}_\text{NCPD}(j)=
\sum^j_{k=1} \lambda'_k \mathbf{a}^{k,1}\otimes \mathbf{a}^{k,2} \otimes \mathbf{a}^{k,3}$.
Similarly, for the results of NTD methods, we also plot
$$\|\mathcal{A}-\mathcal{X}_\text{NTD}(j)\|_F/\|\mathcal{A}\|_F$$ with respect to $j$, where
$\mathcal{X}_\text{NTD}(j) =
[\mathcal{G}\times_1 \mathbf{U}^{(1)}\times_2 \mathbf{U}^{(2)}]_{:,:,\mathbf k_j}
\times_3\mathbf{U}^{(3)}_{:,\mathbf k_j}$, $[\mathcal{G}\times_1 \mathbf{U}^{(1)}\times_2 \mathbf{U}^{(2)}]_{:,:,\mathbf k_j}$ is the $\mathbf k_j$-th mode-12 (spatial) slice of $[\mathcal{G}\times_1 \mathbf{U}^{(1)}\times_2 \mathbf{U}^{(2)}]$, and each $[\mathcal{G}\times_1 \mathbf{U}^{(1)}\times_2 \mathbf{U}^{(2)}]_{:,:,\mathbf k_j}$ is normalized with its Frobenius norm equaling to 1, and
$\mathbf k_j$ indicates a vector composed of the indexes corresponding to the $j$ largest
$\ell_2$ norms of $\mathbf{U}^{(3)}$'s columns.
For the results by our methods, we plot
$$\|\mathcal{A}-\mathcal{X}_\text{NLRT}(j)\|_F/\|\mathcal{A}\|_F$$ with respect to $j$, where $\mathcal{X}_\text{NLRT}(j)=\textrm{fold} \left(\sum_{i=1}^{j}
\sigma_{i}(\mathbf{X}_{3}) {\bf u}_{i}(\mathbf{X}_{3}) {\bf v}_{i}^{T}(\mathbf{X}_{3})\right)$,
$\sigma_{i}(\mathbf{X}_{3})$ is the $i$-th singular values of ${\bf X}_3$,
and ${\mathbf X}_3$ is the third-mode unfolding matrix of $\mathcal{X}$.
The third-mode of ${\cal X}$ is chosen in NTD and our  NLRT, we are interested to observe how many
indices required in the spectral mode of given hyperspectral data.

In Fig. \ref{fig-hsi-sinval}, we can see that when the number of components (namely $j$) increases, the relative residual decreases. Our  NLRT could provide a significant index based on singular values to identify important singular basis vectors for the approximation. Thus, the relative residuals by the proposed  NLRT algorithm are significantly smaller than those by the testing NTD and NCPD algorithms.
Similar phenomena can be found in Fig. \ref{fig-hsi-sinval2}, in which $\mathcal{X}_\text{NTD}(j)$
and $\mathcal{X}_\text{MP-NLRT}(j)$ are computed using the number of indices in the first or second
modes of ${\cal X}$.


\subsection{Image Classification}

The advantage of the proposed NLRT method is that
the important singular basis vectors can be identified within the algorithm. Such basis vectors can
provide useful information for image recognition such as classification.
Here we conduct hyperspectral image classification experiments
on the Indian Pines dataset\footnote{Data available at \scriptsize \url{https://engineering.purdue.edu/$\sim$biehl/MultiSpec/hyperspectral.html}.}.
This data set was captured by the Airborne Visible/Infrared Imaging Spectrometer (AVIRIS) sensor over the Indian Pines test site in North-western Indiana in June 1992.
After removing 20 bands, which cover the region of water absorption, this HSI is of the size $145\times145\times200$.
The ground truth contains 16 land cover classes as shown in Fig. \ref{indianpines_Image}.
Therefore, we set the multilinear rank to be $(16,16,16)$ and the CP rank to be 16
for all the testing methods.
We randomly choose $s$ of the available labeled samples, which are exhibited in
Table \ref{tab-india-samples}. Labeled samples from each class are used
for training and the remaining samples are used for testing.

\begin{table}[!t]
\scriptsize
\renewcommand\arraystretch{0.8}
\setlength{\tabcolsep}{1.5pt}
\centering
\caption{The number of label samples in each class.}
\begin{tabular}{cccccccccccc}\toprule
No. &1&2&3&4&5&6&7&8\\\midrule
\multirow{2}{*}{Name}&\multirow{2}{*}{Alfalfa}&Corn-&Corn-&\multirow{2}{*}{Corn}&Grass-&Grass- &Grass-pasture-&Hay- \\
&&no till&min till&&pasture&trees &mowed&windrowed \\\midrule
Samples&46&1428&830&237&483&730&28&478\\\bottomrule
\\
\toprule
No. &9&10&11&12&13&14&15&16\\\midrule
\multirow{2}{*}{Name}&\multirow{2}{*}{Oat} &Soybean-&Soybean- &Soybean-&\multirow{2}{*}{Wheat} &\multirow{2}{*}{Woods} &Buildings-Grass-&Stone- \\
& &no till&min till &clean& & &Trees-Drives&Steel-Towers \\\midrule
Samples&20&972&2455&593&205&1265&386&93\\\bottomrule
\end{tabular}%
\label{tab-india-samples}
\end{table}
\begin{figure}[!t]
\centering
\begin{tabular}{cc}
\includegraphics[width=0.41\linewidth]{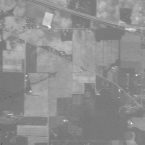} & \includegraphics[width=0.51\linewidth,height=0.41\linewidth]{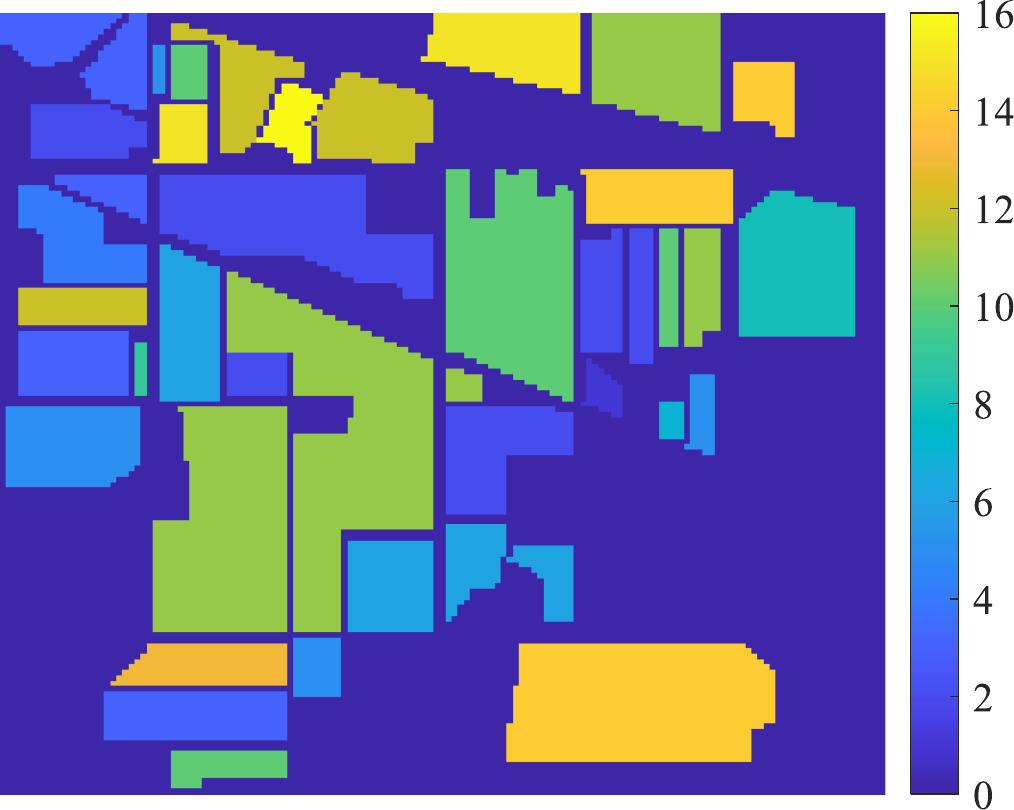}\\
(a) The 10-th band of the original HSI. & (b) The ground truth categorization map.
\end{tabular}
\caption{Indian Pines image and related ground truth categorization information. }
\label{indianpines_Image}
\end{figure}

After obtaining low rank approximations, 16 singular vectors corresponding to the largest 16 singular values of
the unfolding matrix of the tensor approximation along the spectral mode (the third mode) are employed for classification.
We apply the $k$-nearest neighbor ($k$-NN, $k=1,3,5$) classifiers to identify the testing samples in the projected trained samples representation.
The classification accuracy, which is defined as the portion of correctly identified entries, with respect to different $s$ is reported in Table \ref{tab-india-class}.
The results in Table \ref{tab-india-class} show that the classification based on our nonnegative low rank approximation is better than other comparison methods.

\begin{table}[!t]
\renewcommand\arraystretch{0.8}
\setlength{\tabcolsep}{2pt}
\centering
\caption{The accuracy (in terms of percentage) of the classification results on the approximations by different methods. The \textbf{best} values are highlighted in bold.}
\begin{tabular}{cccccccccccccccccc}
\toprule
\multirow{2}{*}{$s$} &Classi-&& \multicolumn{6}{l}{NCPD-}    
 && \multicolumn{4}{l}{NTD-} 
 &&   \multirow{2}{*}{NLRT}   \\

 &  fier                    && MU       & HALS  & APG   & BCD   & CDTF  & SaCD  && MU   & HALS  & APG   & BCD     &&                           \\\midrule
\multirow{3}{*}{10}
& 1-NN & & 69.68 & 69.71 & 67.91 & 66.40 & 65.56 & 61.50 && 65.89 & 71.12 & 73.98 & 73.70 && \bf 74.92 \\
& 3-NN & & 63.79 & 64.72 & 61.89 & 61.52 & 60.57 & 58.00 && 61.65 & 65.25 & 69.80 & 68.02 && \bf 70.12 \\
& 5-NN & & 62.11 & 62.72 & 60.46 & 60.23 & 59.21 & 56.58 && 61.26 & 63.67 & 67.53 & 65.68 && \bf 68.38 \\
\midrule
\multirow{3}{*}{20}
& 1-NN & & 77.04 & 77.35 & 75.05 & 74.78 & 74.74 & 67.95 && 73.14 & 79.21 & 81.16 & 81.51 && \bf 82.06 \\
& 3-NN & & 72.09 & 72.39 & 70.59 & 69.80 & 69.53 & 63.76 && 69.15 & 75.20 & 77.45 & 76.69 && \bf 77.47 \\
& 5-NN & & 69.59 & 70.10 & 68.31 & 68.54 & 67.60 & 63.43 && 67.53 & 73.16 & 75.12 & 74.55 && \bf 75.60 \\
\midrule
\multirow{3}{*}{30}
& 1-NN & & 81.20 & 81.01 & 78.82 & 79.28 & 78.36 & 71.36 && 76.76 & 83.19 & 84.24 & 85.03 && \bf 85.71 \\
& 3-NN & & 76.76 & 76.84 & 74.44 & 74.37 & 73.95 & 68.13 && 72.11 & 78.68 & 80.12 & 80.91 && \bf 81.62 \\
& 5-NN & & 74.06 & 74.52 & 72.38 & 72.21 & 72.14 & 66.46 && 71.18 & 76.54 & 78.29 & 78.74 && \bf 79.16 \\
\midrule
\multirow{3}{*}{40}
& 1-NN & & 84.19 & 84.32 & 81.78 & 82.09 & 82.01 & 74.79 && 79.38 & 86.36 & 86.80 & 87.18 && \bf 88.51 \\
& 3-NN & & 80.17 & 79.96 & 77.99 & 77.99 & 78.14 & 71.17 && 75.49 & 81.76 & 83.93 & 84.11 && \bf 84.87 \\
& 5-NN & & 78.09 & 78.34 & 76.14 & 75.82 & 75.87 & 69.84 && 74.40 & 79.80 & 81.73 & 81.94 && \bf 82.98 \\
\midrule
\multirow{3}{*}{50}
& 1-NN & & 85.73 & 86.27 & 83.50 & 83.89 & 83.81 & 77.15 && 82.14 & 88.09 & 88.16 & 88.81 && \bf 90.19 \\
& 3-NN & & 82.31 & 82.14 & 79.91 & 80.23 & 80.42 & 73.95 && 78.10 & 83.95 & 85.98 & 85.96 && \bf 86.52 \\
& 5-NN & & 80.19 & 80.60 & 77.94 & 78.32 & 78.12 & 72.21 && 76.95 & 81.92 & 84.05 & 84.03 && \bf 84.79 \\
\bottomrule
\end{tabular}%
\label{tab-india-class}
\end{table}
\section{Conclusion}
In the paper, we proposed a new idea for computing nonnegative low rank tensor approximation. We  proposed a method called  NLRT which determines a nonnegative low rank approximation to given data by taking use of low rank matrix manifolds and non-negativity property. The convergence analysis is given. Experiments in synthetic data sets and multi-dimensional image data sets are conducted to present the performance of the proposed  NLRT method. It shows that  NLRT is better than classical nonnegative tensor factorization methods.
\begin{acknowledgements}
\thanks{T.-X. Jiang's research is supported in part by the National Natural Science
Foundation of China under Grant 12001446. M. K. Ng's research is supported in part by the HKRGC GRF under Grant
12300218, 12300519, 17201020 and 17300021. G.-J. Song's research is supported in part by the National Natural Science
Foundation of China under Grant 12171369 and Key NSF of Shandong Province under Grant ZR2020KA008.}
\end{acknowledgements}
\bibliographystyle{spmpsci}
\bibliography{NLRT_refa}

\end{document}